\documentclass[letterpaper]{article} 
\usepackage{aaai24}  

\usepackage[T1]{fontenc}
\usepackage[utf8]{inputenc}
\usepackage{mathtools}

\usepackage{amssymb,mathrsfs}
\usepackage{amsthm}
\usepackage{bm}
\usepackage{scalerel}
\usepackage{nicefrac}
\usepackage{microtype} 
\usepackage[shortlabels]{enumitem}
\usepackage{graphicx}
\usepackage{epstopdf}
\DeclareGraphicsExtensions{.eps,.png,.jpg,.pdf}

\usepackage{colortbl}
\usepackage{booktabs}
\usepackage{multirow}
\usepackage{colortbl,xcolor}
\usepackage{xparse,xstring}
\usepackage{calc}
\usepackage{etoolbox}

\usepackage{array}
\newcolumntype{L}[1]{>{\raggedright\let\newline\\\arraybackslash\hspace{0pt}}m{#1}}
\newcolumntype{C}[1]{>{\centering\let\newline\\\arraybackslash\hspace{0pt}}m{#1}}
\newcolumntype{R}[1]{>{\raggedleft\let\newline\\\arraybackslash\hspace{0pt}}m{#1}}

\makeatletter
\let\MYcaption\@makecaption
\makeatother
\usepackage[font=footnotesize]{subcaption}
\makeatletter
\let\@makecaption\MYcaption
\makeatother

\usepackage{glossaries}
\makeatletter
\sfcode`\.1006

\let\oldgls\gls
\let\oldglspl\glspl

\newcommand\fussy@ifnextchar[3]{%
	\let\reserved@d=#1%
	\def\reserved@a{#2}%
	\def\reserved@b{#3}%
	\futurelet\@let@token\fussy@ifnch}
\def\fussy@ifnch{%
	\ifx\@let@token\reserved@d
		\let\reserved@c\reserved@a
	\else
		\let\reserved@c\reserved@b
	\fi
	\reserved@c}

\renewcommand{\gls}[1]{%
\oldgls{#1}\fussy@ifnextchar.{\@checkperiod}{\@}}
\renewcommand{\glspl}[1]{%
\oldglspl{#1}\fussy@ifnextchar.{\@checkperiod}{\@}}

\newcommand{\@checkperiod}[1]{%
	\ifnum\sfcode`\.=\spacefactor\else#1\fi
}

\robustify{\gls}
\robustify{\glspl}
\makeatother

\newacronym{wrt}{w.r.t.}{with respect to}
\newacronym{RHS}{R.H.S.}{right-hand side}
\newacronym{LHS}{L.H.S.}{left-hand side}
\newacronym{iid}{i.i.d.}{independent and identically distributed}
\newacronym{SOTA}{SOTA}{state-of-the-art}

\usepackage{float}

\usepackage[capitalize]{cleveref}
\crefname{equation}{}{}
\Crefname{equation}{}{}
\crefname{claim}{claim}{claims}
\crefname{step}{step}{steps}
\crefname{line}{line}{lines}
\crefname{condition}{condition}{conditions}
\crefname{dmath}{}{}
\crefname{dseries}{}{}
\crefname{dgroup}{}{}

\crefname{Problem}{Problem}{Problems}
\crefformat{Problem}{Problem~#2#1#3}
\crefrangeformat{Problem}{Problems~#3#1#4 to~#5#2#6}

\crefname{Theorem}{Theorem}{Theorems}
\crefname{Corollary}{Corollary}{Corollaries}
\crefname{Proposition}{Proposition}{Propositions}
\crefname{Lemma}{Lemma}{Lemmas}
\crefname{Definition}{Definition}{Definitions}
\crefname{Example}{Example}{Examples}
\crefname{Assumption}{Assumption}{Assumptions}
\crefname{Remark}{Remark}{Remarks}
\crefname{Rem}{Remark}{Remarks}
\crefname{remarks}{Remarks}{Remarks}
\crefname{Appendix}{Appendix}{Appendices}
\crefname{Supplement}{Supplement}{Supplements}
\crefname{Exercise}{Exercise}{Exercises}
\crefname{Theorem_A}{Theorem}{Theorems}
\crefname{Corollary_A}{Corollary}{Corollaries}
\crefname{Proposition_A}{Proposition}{Propositions}
\crefname{Lemma_A}{Lemma}{Lemmas}
\crefname{Definition_A}{Definition}{Definitions}

\usepackage{crossreftools}

\ifx\loadbreqn\undefined
	\relax
\else
	\usepackage{breqn}
\fi

\makeatletter
\def\cleartheorem#1{%
    \expandafter\let\csname#1\endcsname\relax
    \expandafter\let\csname c@#1\endcsname\relax
}
\def\clearthms#1{ \@for\tname:=#1\do{\cleartheorem\tname} }
\makeatother

\ifx\renewtheorem\undefined
	\ifx\useTheoremCounter\undefined
		\newtheorem{Theorem}{Theorem}
		\newtheorem{Corollary}{Corollary}
		\newtheorem{Proposition}{Proposition}
		
	\else
		\newtheorem{Theorem}{Theorem}

	\fi

	\newtheorem{Definition}{Definition}
	
	\newtheorem{Remark}{Remark}

\fi

\theoremstyle{remark}

\theoremstyle{plain}

\newcommand{\qednew}{\nobreak \ifvmode \relax \else
		\ifdim\lastskip<1.5em \hskip-\lastskip
			\hskip1.5em plus0em minus0.5em \fi \nobreak
		\vrule height0.75em width0.5em depth0.25em\fi}



\newcommand{\nn}{\nonumber\\ }

\NewDocumentCommand{\movedownsub}{e{^_}}{%
	\IfNoValueTF{#1}{%
		\IfNoValueF{#2}{^{}}
	}{%
		^{#1}
	}%
	\IfNoValueF{#2}{_{#2}}
}

\let\latexchi\chi
\RenewDocumentCommand{\chi}{}{\latexchi\movedownsub}

\newcommand{\Real}{\mathbb{R}}

\newcommand{\Complex}{\mathbb{C}}

\newcommand{\calF}{\mathcal{F}}
\newcommand{\calG}{\mathcal{G}}

\newcommand{\calV}{\mathcal{V}}

\newcommand{\bW}{\mathbf{W}}
\newcommand{\bx}{\mathbf{x}}
\newcommand{\bX}{\mathbf{X}}

\DeclareSymbolFont{bsfletters}{OT1}{cmss}{bx}{n}
\DeclareSymbolFont{ssfletters}{OT1}{cmss}{m}{n}
\DeclareMathSymbol{\bsfGamma}{0}{bsfletters}{'000}
\DeclareMathSymbol{\ssfGamma}{0}{ssfletters}{'000}
\DeclareMathSymbol{\bsfDelta}{0}{bsfletters}{'001}
\DeclareMathSymbol{\ssfDelta}{0}{ssfletters}{'001}
\DeclareMathSymbol{\bsfTheta}{0}{bsfletters}{'002}
\DeclareMathSymbol{\ssfTheta}{0}{ssfletters}{'002}
\DeclareMathSymbol{\bsfLambda}{0}{bsfletters}{'003}
\DeclareMathSymbol{\ssfLambda}{0}{ssfletters}{'003}
\DeclareMathSymbol{\bsfXi}{0}{bsfletters}{'004}
\DeclareMathSymbol{\ssfXi}{0}{ssfletters}{'004}
\DeclareMathSymbol{\bsfPi}{0}{bsfletters}{'005}
\DeclareMathSymbol{\ssfPi}{0}{ssfletters}{'005}
\DeclareMathSymbol{\bsfSigma}{0}{bsfletters}{'006}
\DeclareMathSymbol{\ssfSigma}{0}{ssfletters}{'006}
\DeclareMathSymbol{\bsfUpsilon}{0}{bsfletters}{'007}
\DeclareMathSymbol{\ssfUpsilon}{0}{ssfletters}{'007}
\DeclareMathSymbol{\bsfPhi}{0}{bsfletters}{'010}
\DeclareMathSymbol{\ssfPhi}{0}{ssfletters}{'010}
\DeclareMathSymbol{\bsfPsi}{0}{bsfletters}{'011}
\DeclareMathSymbol{\ssfPsi}{0}{ssfletters}{'011}
\DeclareMathSymbol{\bsfOmega}{0}{bsfletters}{'012}
\DeclareMathSymbol{\ssfOmega}{0}{ssfletters}{'012}

\makeatletter
\newcommand*\rel@kern[1]{\kern#1\dimexpr\macc@kerna}
\newcommand*\widebar[1]{%
  \begingroup
  \def\mathaccent##1##2{%
    \rel@kern{0.8}%
    \overline{\rel@kern{-0.8}\macc@nucleus\rel@kern{0.2}}%
    \rel@kern{-0.2}%
  }%
  \macc@depth\@ne
  \let\math@bgroup\@empty \let\math@egroup\macc@set@skewchar
  \mathsurround\z@ \frozen@everymath{\mathgroup\macc@group\relax}%
  \macc@set@skewchar\relax
  \let\mathaccentV\macc@nested@a
  \macc@nested@a\relax111{#1}%
  \endgroup
}
\makeatother

\DeclareMathOperator{\diag}{diag}

\DeclareMathOperator{\var}{var}

\DeclareMathOperator{\cov}{cov}

\newcommand{\ifbcdot}[1]{\ifblank{#1}{\cdot}{#1}}

\DeclarePairedDelimiterX\abs[1]{\lvert}{\rvert}{\ifbcdot{#1}}
\DeclarePairedDelimiterX\parens[1]{(}{)}{\ifbcdot{#1}}
\DeclarePairedDelimiterX\brk[1]{[}{]}{\ifbcdot{#1}}
\DeclarePairedDelimiterX\braces[1]{\{}{\}}{\ifbcdot{#1}}
\DeclarePairedDelimiterX\angles[1]{\langle}{\rangle}{\ifblank{#1}{\cdot,\cdot}{#1}}
\DeclarePairedDelimiterX\ip[2]{\langle}{\rangle}{\ifbcdot{#1},\ifbcdot{#2}}
\DeclarePairedDelimiterX\norm[1]{\lVert}{\rVert}{\ifbcdot{#1}}
\DeclarePairedDelimiterX\ceil[1]{\lceil}{\rceil}{\ifbcdot{#1}}
\DeclarePairedDelimiterX\floor[1]{\lfloor}{\rfloor}{\ifbcdot{#1}}

\DeclareFontFamily{U}{matha}{\hyphenchar\font45}
\DeclareFontShape{U}{matha}{m}{n}{
      <5> <6> <7> <8> <9> <10> gen * matha
      <10.95> matha10 <12> <14.4> <17.28> <20.74> <24.88> matha12
      }{}
\DeclareSymbolFont{matha}{U}{matha}{m}{n}
\DeclareFontSubstitution{U}{matha}{m}{n}

\DeclareFontFamily{U}{mathx}{\hyphenchar\font45}
\DeclareFontShape{U}{mathx}{m}{n}{
      <5> <6> <7> <8> <9> <10>
      <10.95> <12> <14.4> <17.28> <20.74> <24.88>
      mathx10
      }{}
\DeclareSymbolFont{mathx}{U}{mathx}{m}{n}
\DeclareFontSubstitution{U}{mathx}{m}{n}

\DeclareMathDelimiter{\vvvert}{0}{matha}{"7E}{mathx}{"17}
\DeclarePairedDelimiterX\vertiii[1]{\vvvert}{\vvvert}{\ifbcdot{#1}}

\DeclarePairedDelimiterXPP\trace[1]{\operatorname{Tr}}{(}{)}{}{\ifbcdot{#1}} 
\DeclarePairedDelimiterXPP\col[1]{\operatorname{col}}{\{}{\}}{}{\ifbcdot{#1}} 
\DeclarePairedDelimiterXPP\row[1]{\operatorname{row}}{\{}{\}}{}{\ifbcdot{#1}} 
\DeclarePairedDelimiterXPP\erf[1]{\operatorname{erf}}{(}{)}{}{\ifbcdot{#1}}
\DeclarePairedDelimiterXPP\erfc[1]{\operatorname{erfc}}{(}{)}{}{\ifbcdot{#1}}
\DeclarePairedDelimiterXPP\KLD[2]{D}{(}{)}{}{\ifbcdot{#1}\, \delimsize\|\, \ifbcdot{#2}} 
\DeclarePairedDelimiterXPP\op[2]{\operatorname{#1}}{(}{)}{}{#2} 

\newcommand{\T}{^{\mkern-1.5mu\mathop\intercal}}
\newcommand{\ud}{\,\mathrm{d}} 

\DeclarePairedDelimiterXPP\indicate[1]{{\bf 1}}{\{}{\}}{}{\ifbcdot{#1}}

\NewDocumentCommand\ofrac{s m}{%
	\IfBooleanTF#1%
	{\dfrac{1}{#2}}%
	{\frac{1}{#2}}%
}
\NewDocumentCommand\ddfrac{s m m}{%
	\IfBooleanTF#1%
	{\dfrac{\mathrm{d} {#2}}{\mathrm{d} {#3}}}%
	{\frac{\mathrm{d} {#2}}{\mathrm{d} {#3}}}%
}
\NewDocumentCommand\ppfrac{s m m}{%
	\IfBooleanTF#1%
	{\dfrac{\partial {#2}}{\partial {#3}}}%
	{\frac{\partial {#2}}{\partial {#3}}}%
}

\providecommand\given{}

\DeclarePairedDelimiterX\Set[2]\{\}{%
\renewcommand\given{\SetSymbol[\delimsize]{#1}}
#2
}
\DeclarePairedDelimiterX\Setc[1]\{\}{%
\renewcommand\given{\SetSymbol{:}}
#1
}

\NewDocumentCommand\set{s o m}{%
	\IfBooleanTF#1%
	{\IfValueTF{#2}{\Set*{#2}{#3}}{\Setc*{#3}}}%
	{\IfValueTF{#2}{\Set{#2}{#3}}{\Setc{#3}}}%
}

\NewDocumentCommand{\evalat}{ s O{\big} m e{_^} }{%
\IfBooleanTF{#1}%
{\left. #3 \right|}{#3#2|}%
\IfValueT{#4}{_{#4}}%
\IfValueT{#5}{^{#5}}%
}

\providecommand\given{}
\DeclarePairedDelimiterXPP\cprob[1]{}(){}{
\renewcommand\given{\nonscript\,\delimsize\vert\allowbreak\nonscript\,\mathopen{}}%
#1%
}
\DeclarePairedDelimiterXPP\cexp[1]{}[]{}{
\renewcommand\given{\nonscript\,\delimsize\vert\allowbreak\nonscript\,\mathopen{}}%
#1%
}

\DeclareDocumentCommand \P { s e{_^} d() g } {%
	\mathbb{P}%
	\IfBooleanTF{#1}%
		{
			\IfValueT{#2}{_{#2}}%
			\IfValueT{#3}{^{#3}}%
			\IfValueTF{#5}{\cprob{#4 \given #5}}{\IfValueT{#4}{\cprob{#4}}}%
		}%
		{
			\IfValueT{#2}{_{#2}}%
			\IfValueT{#3}{^{#3}}%
			\IfValueTF{#5}{\cprob*{#4 \given #5}}{\IfValueT{#4}{\cprob*{#4}}}%
		}%
}

\DeclareDocumentCommand \E { s e{_^} o g } {%
	\mathbb{E}%
	\IfBooleanTF{#1}%
		{
			\IfValueT{#2}{_{#2}}%
			\IfValueT{#3}{^{#3}}%
			\IfValueTF{#5}{\cexp{#4 \given #5}}{\IfValueT{#4}{\cexp{#4}}}%
		}%
		{
			\IfValueT{#2}{_{#2}}%
			\IfValueT{#3}{^{#3}}%
			\IfValueTF{#5}{\cexp*{#4 \given #5}}{\IfValueT{#4}{\cexp*{#4}}}%
		}%
}

\DeclareDocumentCommand \Var { s e{_^} d() g } {%
	\var%
	\IfBooleanTF{#1}%
		{
			\IfValueT{#2}{_{#2}}%
			\IfValueT{#3}{^{#3}}%
			\IfValueTF{#5}{\cprob{#4 \given #5}}{\IfValueT{#4}{\cprob{#4}}}%
		}%
		{
			\IfValueT{#2}{_{#2}}%
			\IfValueT{#3}{^{#3}}%
			\IfValueTF{#5}{\cprob*{#4 \given #5}}{\IfValueT{#4}{\cprob*{#4}}}%
		}%
}

\DeclareDocumentCommand \Cov { s e{_^} d() g } {%
	\cov%
	\IfBooleanTF{#1}%
		{
			\IfValueT{#2}{_{#2}}%
			\IfValueT{#3}{^{#3}}%
			\IfValueTF{#5}{\cprob{#4 \given #5}}{\IfValueT{#4}{\cprob{#4}}}%
		}%
		{
			\IfValueT{#2}{_{#2}}%
			\IfValueT{#3}{^{#3}}%
			\IfValueTF{#5}{\cprob*{#4 \given #5}}{\IfValueT{#4}{\cprob*{#4}}}%
		}%
}

\ExplSyntaxOn
\NewDocumentCommand \dist {m o o} {%
\mathrm{#1}\left(%
	\IfValueT{#3}{%
		\tl_if_blank:nTF{ #3 }{\cdot\, \middle|\, }{#3\, \middle|\, }%
	}
	\IfValueT{#2}{#2}%
\right)%
}
\ExplSyntaxOff

\NewDocumentCommand {\cbrace} {t+ D[]{black} D(){\widthof{#5}} m m } {%
	\begingroup%
		\color{#2}
		\IfBooleanTF{#1}{%
			\overbrace{#4}^%
		}{
			\underbrace{#4}_%
		}%
		{\parbox[c]{#3}{\centering\footnotesize{#5}}}%
	\endgroup%
}

\let\oldforall\forall
\renewcommand{\forall}{\oldforall \, }

\let\oldexist\exists
\renewcommand{\exists}{\oldexist \, }

\makeatletter

\newcommand{\rankcolor}[2]{%
	\expandafter\renewcommand\csname #1\endcsname[1]{%
		\ifblank{##1}{%
			{\color{#2} \textbf{#2}}%
		}{%
			\ifmmode
				\textcolor{#2}{\bm{##1}}%
			\else%
				{\color{#2} \textbf{##1}}%
			\fi	
		}%
	}
}

\graphicspath{{./Figures/}{./figures/}}
\DeclareDocumentCommand{\includeCroppedPdf}{ o O{./Figures/} m }{
	\IfFileExists{#2#3-crop.pdf}{}{%
		\immediate\write18{pdfcrop #2#3.pdf #2#3-crop.pdf}}%
	\includegraphics[#1]{#2#3-crop.pdf}
}

\makeatletter
\newcommand*{\addFileDependency}[1]{
  \typeout{(#1)}
  \@addtofilelist{#1}
  \IfFileExists{#1}{}{\typeout{No file #1.}}
}
\makeatother

\definecolor{gray90}{gray}{0.9}
\def\colorlist{red,blue,brown,cyan,darkgray,gray,lightgray,green,lime,magenta,olive,orange,pink,purple,teal,violet,white,yellow}

\makeatletter
\def\startcomment{[}
\ifx\nohighlights\undefined
	\newcommand{\createcolor}[1]{%
			\expandafter\newcommand\csname #1\endcsname[1]{{\color{#1} ##1}}%
	}
	\newcommand{\msout}[1]{\text{\color{green} \sout{\ensuremath{#1}}}}
	\newcommand{\del}[1]{{\color{green}\ifmmode \msout{#1}\else\sout{#1}\fi}}
\else
	\newcommand{\createcolor}[1]{%
			\expandafter\newcommand\csname #1\endcsname[1]{%
				\noexpandarg%
				\StrChar{##1}{1}[\firstletter]%
				\if\firstletter\startcomment%
					\relax
				\else%
					##1
				\fi
			}%
	}
	\newcommand{\msout}[1]{}
	\newcommand{\del}[1]{}
\fi

\def\@tempa#1,{%
    \ifx\relax#1\relax\else
        \createcolor{#1}%
        \expandafter\@tempa
    \fi
}
\expandafter\@tempa\colorlist,\relax,
\makeatother

\newcommand{\hhide}[1]{}

\ifx\diagnoselabel\undefined
	\relax
\else
	\makeatletter
	\def\@testdef #1#2#3{%
		\def\reserved@a{#3}\expandafter \ifx \csname #1@#2\endcsname
			\reserved@a  \else
			\typeout{^^Jlabel #2 changed:^^J%
				\meaning\reserved@a^^J%
				\expandafter\meaning\csname #1@#2\endcsname^^J}%
			\@tempswatrue \fi}
	\makeatother
\fi

\newcommand{\tb}[1]{\textbf{#1}}

\usepackage{times}  
\usepackage{helvet}  
\usepackage{courier}  
\usepackage[hyphens]{url}  
\usepackage{natbib}  
\usepackage{caption} 
\usepackage{algorithm}
\usepackage{algorithmic}

\usepackage{tikz}
\usepackage{adjustbox}
\usepackage{newfloat}
\usepackage{listings}
\DeclareCaptionStyle{ruled}{labelfont=normalfont,labelsep=colon,strut=off} 
\floatstyle{ruled}
\newfloat{listing}{tb}{lst}{}
\floatname{listing}{Listing}

\newcommand{\first}{\textbf}
\newcommand{\second}{\underline}
\title{Coupling Graph Neural Networks with Fractional Order\\ 
Continuous Dynamics: A Robustness Study}
\author{
    Qiyu~Kang\textsuperscript{\rm 1}\equalcontrib, Kai~Zhao\textsuperscript{\rm 1}\equalcontrib\thanks{Correspondence to: Kai Zhao <kai.zhao@ntu.edu.sg>}, 
    Yang~Song\textsuperscript{\rm 2}, Yihang Xie\textsuperscript{\rm 1},\\ Yanan Zhao\textsuperscript{\rm 1}, Sijie Wang\textsuperscript{\rm 1}, Rui She\textsuperscript{\rm 1}, Wee~Peng~Tay\textsuperscript{\rm 1}
}
\affiliations{
\textsuperscript{\rm 1}Nanyang Technological University\\
\textsuperscript{\rm 2}C3 AI, Singapore\\
}

\begin{document}

\maketitle

\begin{abstract}
In this work, we rigorously investigate the robustness of graph neural fractional-order differential equation (FDE) models. This framework extends beyond traditional graph neural (integer-order) ordinary differential equation (ODE) models by implementing the time-fractional Caputo derivative. Utilizing fractional calculus allows our model to consider long-term memory during the feature updating process, diverging from the memoryless Markovian updates seen in traditional graph neural ODE models. 
The superiority of graph neural FDE models over graph neural ODE models has been established in environments free from attacks or perturbations.
While traditional graph neural ODE models have been verified to possess a degree of stability and resilience in the presence of adversarial attacks in existing literature, the robustness of graph neural FDE models, especially under adversarial conditions, remains largely unexplored. This paper undertakes a detailed assessment of the robustness of graph neural FDE models. We establish a theoretical foundation outlining the robustness characteristics of graph neural FDE models, highlighting that they maintain more stringent output perturbation bounds in the face of input and graph topology disturbances, compared to their integer-order counterparts. Our empirical evaluations further confirm the enhanced robustness of graph neural FDE models, highlighting their potential in adversarially robust applications.

\end{abstract}
\section{Introduction}
Graph Neural Networks (GNNs) \cite{ kipf2017semi, velickovic2017graph,JiLeeMen:C23,LeeJiTay:C22,SheKanWan:J23} have emerged as an influential tool capable of extracting meaningful representations from intricate datasets, such as social networks \cite{huang2021social} and molecular structures \cite{guo2022molecular}. 
Despite their impressive capability, GNNs have been found susceptible to adversarial attacks \cite{dai2018advgnn,ma2020pracadvgnn,zugnernettack}, with modifications or injections into the graph often causing significant degradation in performance.
In real-world scenarios, it is common for data to be perturbed during the training or testing phases \cite{dai2022comprehensive,wang2019fraud}, highlighting the importance of studying the robustness of GNNs.  For instance, in financial systems, fraudulent activities may introduce slight perturbations into transactional data, making it paramount for the underlying models to remain robust against these adversarial changes. Similarly, in social networks, misinformation or the presence of bots can skew the data, which can subsequently impact the insights drawn from it. Therefore, the robustness of GNNs is not just a theoretical concern but a practical necessity. 
Several defensive strategies have been established to counteract the damaging implications of adversarial attacks on graph data. Approaches such as GARNET \cite{deng2022garnet}, GNN-Guard \cite{zhang2020gnnguard}, RGCN \cite{zhu2019robustgcn}, and Pro-GNN \cite{jin2020prognn} are grounded in preprocessing techniques that aim to remove adversarial alterations to the structure before GNN training commences. Nonetheless, these methods often necessitate the exploration of graph structure properties, leading to higher computational costs. Furthermore, these strategies are more suitably tailored to combat poisoning attacks.

Recent advances have witnessed a growing use of dynamical system theory in designing and understanding GNNs. Models like CGNN \cite{xhonneux2020continuous}, GRAND \cite{chamberlain2021grand}, GRAND++ \cite{thorpe2021grand++}, GraphCON \cite{rusch2022icml}, HANG \cite{ZhaKanSon:C23b} and CDE \cite{zhao2023cde} employ ordinary differential equations (ODEs) to offer a dynamical system perspective on graph node feature evolution. Typically, these dynamics can be described by:
\begin{align}
\frac{\ud \bX(t)}{\ud t} = \calF(\bW,\bX(t)). \label{eq.graphode}
\end{align}
In this formulation, $\bX(t)$ represents the evolving node features with $\bX(0)$ as the initial input node features, while $\bW$ is the graph's adjacency matrix. 
The function, $\calF$, is specifically tailored for graph dynamics. As a case in point, GRAND \cite{chamberlain2021grand} deploys an attention-based aggregation mechanism akin to heat diffusion on the graph. 
Motivated by the Beltrami diffusion equation \cite{sochenTIP1998beltrami}, the paper \cite{SonKanWan:C22}  introduces a model based on the Beltrami flow (abbreviated as GraphBel) and designed for enhanced robustness, particularly in the face of topological perturbations.
In the study \cite{ZhaKanSon:C23b}, graph feature updates are conceptualized as a Hamiltonian flow, endowed with Lyapunov stability, to effectively counter adversarial perturbations. 
GraphCON \cite{rusch2022icml} presents a approach by introducing a second-order graph coupled oscillator for modeling feature updates.
This model can be decomposed into two first-order equations, aligning with the principle that higher integer-order ODEs can be expressed as a system of first-order ODEs through auxiliary variables, effectively encapsulated in \cref{eq.graphode}.

Recent studies have ventured into the intersection of GNNs and fractional calculus \cite{diethelm2010frationalde,kang2023advancing}. One prominent example is the FRactional-Order graph Neural Dynamical network (FROND) framework \cite{FROND2023}. Distinct from conventional graph neural ODE models, FROND leverages fractional-order differential equations (FDEs), with dynamics represented as:
\begin{align}
{D}_t^\beta \bX(t) = \calF(\bW,\bX(t)),\ \beta>0. \label{eq.frac_gra_difaaintro}
\end{align}
The function $\calF(\bW,\bX(t))$ maintains its form as in \cref{eq.graphode}. Typically, we set $\beta\in (0,1]$. The Caputo fractional derivative, denoted by ${D}_t^\beta$, infuses memory into the temporal dynamics (see \cref{ssec.fde} for more details). For $\beta = 1$, the equation reverts to the familiar first-order dynamics as in \cref{eq.graphode}. The distinction lies in the fact that the conventional integer-order derivative measures the function’s \emph{instantaneous change rate}, concentrating on the proximate vicinity of the point. \emph{In contrast, the fractional-order derivative is influenced by the entire historical trajectory of the function,} which substantially diverges from the localized impact found in integer-order derivatives. 

Incorporating a fractional derivative provides GNNs an avenue to mitigate the prevalent oversmoothing problems by enabling slow algebraic convergence \cite{FROND2023}, different from the standard fast exponential convergence. 
Further, with the integration of fractional dynamics, FROND can effortlessly merge with existing graph neural ODE frameworks, potentially increasing their effectiveness, especially with diverse $\beta$ values, without incorporating any additional training parameters to the underlying graph neural ODE models.
Critically, $\beta$ acts as a proxy for the extent of memory in the feature dynamics: a value of $\beta=1$ corresponds to memoryless Markovian dynamics, while $\beta<1$ denotes non-Markovian dynamics with memory. This nuance is further visualized in \cref{fig.block}, where a $\beta<1$ signifies nontrivial dense connections across model discretization timestamps.

Though FROND showcases proficiency in decoding complex graph data patterns, its robustness against adversarial perturbations remains an area of exploration.
By broadening the order of time derivatives from integers to real numbers, fractional calculus can encapsulate more intricate dynamics and data relationships, such as long-range memory effects, where the system's current state is influenced by its comprehensive history, not merely its recent states. 
This capability augments a GNN's ability to more accurately represent the node features across layers, rendering them less susceptible to noise and perturbations. 
In this work, we delve deeply into the ramifications of the fractional order parameter $\beta$ on the robustness attributes of FROND. Our analysis suggests a monotonic relationship between the model's perturbation bounds and the parameter $\beta$, with smaller $\beta$ values indicating augmented robustness.

Our contributions are summarized as follows:
\begin{itemize}
\item We rigorously investigate the robustness characteristics of graph neural FDE models, i.e., FROND models. We show that FROND models exhibit tighter output perturbation bounds compared to their integer-order counterparts in the presence of input and topology perturbations.
\item Through extensive experimental evaluations, including graph modifications and injection attacks,  we empirically demonstrate the superior robustness of FROND models in contrast to conventional graph neural ODE models.
\end{itemize}

\begin{figure}[t]
    \centering
    \adjustbox{scale=0.8,center}{

\tikzset{every picture/.style={line width=0.75pt}} 

\begin{tikzpicture}[x=0.75pt,y=0.75pt,yscale=-1,xscale=1]

\draw  [color={rgb, 255:red, 155; green, 155; blue, 155 }  ,draw opacity=1 ][fill={rgb, 255:red, 255; green, 255; blue, 255 }  ,fill opacity=1 ] (190.39,67.6) -- (316.6,45.3) -- (315.25,98.07) -- (189.03,120.37) -- cycle ;
\draw  [fill={rgb, 255:red, 74; green, 144; blue, 226 }  ,fill opacity=1 ] (287.1,76.9) .. controls (286.37,75.4) and (287.66,73.85) .. (289.98,73.44) .. controls (292.31,73.02) and (294.79,73.89) .. (295.52,75.39) .. controls (296.25,76.88) and (294.96,78.43) .. (292.64,78.85) .. controls (290.32,79.27) and (287.84,78.39) .. (287.1,76.9) -- cycle ;
\draw   (239.42,85.71) .. controls (238.69,84.22) and (239.97,82.66) .. (242.3,82.25) .. controls (244.62,81.83) and (247.1,82.7) .. (247.84,84.2) .. controls (248.57,85.69) and (247.28,87.24) .. (244.96,87.66) .. controls (242.63,88.08) and (240.15,87.2) .. (239.42,85.71) -- cycle ;
\draw    (240.13,86.9) -- (219.11,100.76) ;
\draw    (210.62,81.14) -- (238.8,83.77) ;
\draw [color={rgb, 255:red, 2; green, 122; blue, 250 }  ,draw opacity=0.4 ]   (249.8,83.83) -- (262.59,81.47) -- (264.6,81.09) -- (287.23,76.91) ;
\draw [shift={(247.84,84.2)}, rotate = 349.51] [fill={rgb, 255:red, 2; green, 122; blue, 250 }  ,fill opacity=0.4 ][line width=0.08]  [draw opacity=0] (7.2,-1.8) -- (0,0) -- (7.2,1.8) -- cycle    ;
\draw   (202.21,82.65) .. controls (201.47,81.16) and (202.76,79.61) .. (205.09,79.19) .. controls (207.41,78.77) and (209.89,79.64) .. (210.62,81.14) .. controls (211.36,82.63) and (210.07,84.18) .. (207.74,84.6) .. controls (205.42,85.02) and (202.94,84.15) .. (202.21,82.65) -- cycle ;
\draw   (211.79,102.98) .. controls (211.05,101.48) and (212.34,99.93) .. (214.67,99.51) .. controls (216.99,99.1) and (219.47,99.97) .. (220.2,101.46) .. controls (220.94,102.96) and (219.65,104.51) .. (217.32,104.93) .. controls (215,105.34) and (212.52,104.47) .. (211.79,102.98) -- cycle ;
\draw [color={rgb, 255:red, 0; green, 0; blue, 0 }  ,draw opacity=1 ]   (192,149.5) -- (349.34,254.39) ;
\draw [shift={(351,255.5)}, rotate = 213.69] [color={rgb, 255:red, 0; green, 0; blue, 0 }  ,draw opacity=1 ][line width=0.75]    (10.93,-3.29) .. controls (6.95,-1.4) and (3.31,-0.3) .. (0,0) .. controls (3.31,0.3) and (6.95,1.4) .. (10.93,3.29)   ;
\draw  [color={rgb, 255:red, 155; green, 155; blue, 155 }  ,draw opacity=1 ][fill={rgb, 255:red, 255; green, 255; blue, 255 }  ,fill opacity=1 ] (233.99,97.3) -- (360.2,75) -- (358.85,127.77) -- (232.63,150.07) -- cycle ;
\draw  [fill={rgb, 255:red, 74; green, 144; blue, 226 }  ,fill opacity=1 ] (330.7,106.1) .. controls (329.97,104.6) and (331.26,103.05) .. (333.58,102.64) .. controls (335.91,102.22) and (338.39,103.09) .. (339.12,104.59) .. controls (339.85,106.08) and (338.56,107.63) .. (336.24,108.05) .. controls (333.92,108.47) and (331.44,107.59) .. (330.7,106.1) -- cycle ;
\draw   (283.02,114.91) .. controls (282.29,113.42) and (283.57,111.86) .. (285.9,111.45) .. controls (288.22,111.03) and (290.7,111.9) .. (291.44,113.4) .. controls (292.17,114.89) and (290.88,116.44) .. (288.56,116.86) .. controls (286.23,117.28) and (283.75,116.4) .. (283.02,114.91) -- cycle ;
\draw    (283.73,116.1) -- (262.71,129.96) ;
\draw    (254.22,110.34) -- (282.4,112.97) ;
\draw [color={rgb, 255:red, 2; green, 122; blue, 250 }  ,draw opacity=0.4 ]   (293.4,113.03) -- (306.19,110.67) -- (308.2,110.29) -- (330.83,106.11) ;
\draw [shift={(291.44,113.4)}, rotate = 349.51] [fill={rgb, 255:red, 2; green, 122; blue, 250 }  ,fill opacity=0.4 ][line width=0.08]  [draw opacity=0] (7.2,-1.8) -- (0,0) -- (7.2,1.8) -- cycle    ;
\draw   (245.81,111.85) .. controls (245.07,110.36) and (246.36,108.81) .. (248.69,108.39) .. controls (251.01,107.97) and (253.49,108.84) .. (254.22,110.34) .. controls (254.96,111.83) and (253.67,113.38) .. (251.34,113.8) .. controls (249.02,114.22) and (246.54,113.35) .. (245.81,111.85) -- cycle ;
\draw   (255.39,132.18) .. controls (254.65,130.68) and (255.94,129.13) .. (258.27,128.71) .. controls (260.59,128.3) and (263.07,129.17) .. (263.8,130.66) .. controls (264.54,132.16) and (263.25,133.71) .. (260.92,134.13) .. controls (258.6,134.54) and (256.12,133.67) .. (255.39,132.18) -- cycle ;
\draw  [color={rgb, 255:red, 155; green, 155; blue, 155 }  ,draw opacity=1 ][fill={rgb, 255:red, 255; green, 255; blue, 255 }  ,fill opacity=1 ] (279.59,126.1) -- (405.8,103.8) -- (404.45,156.57) -- (278.23,178.87) -- cycle ;
\draw  [fill={rgb, 255:red, 74; green, 144; blue, 226 }  ,fill opacity=1 ] (376.7,134.5) .. controls (375.97,133) and (377.26,131.45) .. (379.58,131.04) .. controls (381.91,130.62) and (384.39,131.49) .. (385.12,132.99) .. controls (385.85,134.48) and (384.56,136.03) .. (382.24,136.45) .. controls (379.92,136.87) and (377.44,135.99) .. (376.7,134.5) -- cycle ;
\draw   (329.02,143.31) .. controls (328.29,141.82) and (329.57,140.26) .. (331.9,139.85) .. controls (334.22,139.43) and (336.7,140.3) .. (337.44,141.8) .. controls (338.17,143.29) and (336.88,144.84) .. (334.56,145.26) .. controls (332.23,145.68) and (329.75,144.8) .. (329.02,143.31) -- cycle ;
\draw    (329.73,144.5) -- (308.71,158.36) ;
\draw    (300.22,138.74) -- (328.4,141.37) ;
\draw [color={rgb, 255:red, 2; green, 122; blue, 250 }  ,draw opacity=0.4 ]   (339.4,141.43) -- (352.19,139.07) -- (354.2,138.69) -- (376.83,134.51) ;
\draw [shift={(337.44,141.8)}, rotate = 349.51] [fill={rgb, 255:red, 2; green, 122; blue, 250 }  ,fill opacity=0.4 ][line width=0.08]  [draw opacity=0] (7.2,-1.8) -- (0,0) -- (7.2,1.8) -- cycle    ;
\draw   (291.81,140.25) .. controls (291.07,138.76) and (292.36,137.21) .. (294.69,136.79) .. controls (297.01,136.37) and (299.49,137.24) .. (300.22,138.74) .. controls (300.96,140.23) and (299.67,141.78) .. (297.34,142.2) .. controls (295.02,142.62) and (292.54,141.75) .. (291.81,140.25) -- cycle ;
\draw   (301.39,160.58) .. controls (300.65,159.08) and (301.94,157.53) .. (304.27,157.11) .. controls (306.59,156.7) and (309.07,157.57) .. (309.8,159.06) .. controls (310.54,160.56) and (309.25,162.11) .. (306.92,162.53) .. controls (304.6,162.94) and (302.12,162.07) .. (301.39,160.58) -- cycle ;
\draw [color={rgb, 255:red, 4; green, 116; blue, 248 }  ,draw opacity=0.4 ]   (339.12,104.59) .. controls (378.42,112.07) and (376.4,117.04) .. (382.32,132.12) ;
\draw [shift={(383,133.8)}, rotate = 247.48] [fill={rgb, 255:red, 4; green, 116; blue, 248 }  ,fill opacity=0.4 ][line width=0.08]  [draw opacity=0] (7.2,-1.8) -- (0,0) -- (7.2,1.8) -- cycle    ;
\draw [color={rgb, 255:red, 4; green, 116; blue, 248 }  ,draw opacity=0.4 ]   (295.52,75.39) .. controls (332.84,84.19) and (330.87,88.96) .. (336.19,101.22) ;
\draw [shift={(337,103)}, rotate = 244.8] [fill={rgb, 255:red, 4; green, 116; blue, 248 }  ,fill opacity=0.4 ][line width=0.08]  [draw opacity=0] (7.2,-1.8) -- (0,0) -- (7.2,1.8) -- cycle    ;
\draw  [color={rgb, 255:red, 155; green, 155; blue, 155 }  ,draw opacity=1 ][fill={rgb, 255:red, 255; green, 255; blue, 255 }  ,fill opacity=1 ] (324.39,154.6) -- (450.6,132.3) -- (449.25,185.07) -- (323.03,207.37) -- cycle ;
\draw  [color={rgb, 255:red, 155; green, 155; blue, 155 }  ,draw opacity=1 ][fill={rgb, 255:red, 255; green, 255; blue, 255 }  ,fill opacity=1 ] (328.39,157.6) -- (454.6,135.3) -- (453.25,188.07) -- (327.03,210.37) -- cycle ;
\draw  [color={rgb, 255:red, 155; green, 155; blue, 155 }  ,draw opacity=1 ][fill={rgb, 255:red, 255; green, 255; blue, 255 }  ,fill opacity=1 ] (333.39,160.1) -- (459.6,137.8) -- (458.25,190.57) -- (332.03,212.87) -- cycle ;
\draw  [color={rgb, 255:red, 155; green, 155; blue, 155 }  ,draw opacity=1 ][fill={rgb, 255:red, 255; green, 255; blue, 255 }  ,fill opacity=1 ] (337.77,163.45) -- (463.99,141.15) -- (462.63,193.91) -- (336.42,216.21) -- cycle ;
\draw  [color={rgb, 255:red, 155; green, 155; blue, 155 }  ,draw opacity=1 ][fill={rgb, 255:red, 255; green, 255; blue, 255 }  ,fill opacity=1 ] (342.27,165.95) -- (468.49,143.65) -- (467.13,196.41) -- (340.92,218.71) -- cycle ;
\draw  [color={rgb, 255:red, 155; green, 155; blue, 155 }  ,draw opacity=1 ][fill={rgb, 255:red, 255; green, 255; blue, 255 }  ,fill opacity=1 ] (347.27,169.45) -- (473.49,147.15) -- (472.13,199.91) -- (345.92,222.21) -- cycle ;
\draw  [color={rgb, 255:red, 155; green, 155; blue, 155 }  ,draw opacity=1 ][fill={rgb, 255:red, 255; green, 255; blue, 255 }  ,fill opacity=1 ] (353.29,175.06) -- (479.5,152.76) -- (478.15,205.53) -- (351.93,227.83) -- cycle ;
\draw  [fill={rgb, 255:red, 74; green, 144; blue, 226 }  ,fill opacity=1 ] (450,183.86) .. controls (449.27,182.36) and (450.56,180.81) .. (452.88,180.4) .. controls (455.21,179.98) and (457.69,180.85) .. (458.42,182.35) .. controls (459.15,183.84) and (457.86,185.39) .. (455.54,185.81) .. controls (453.22,186.23) and (450.74,185.35) .. (450,183.86) -- cycle ;
\draw   (402.32,192.67) .. controls (401.59,191.17) and (402.87,189.62) .. (405.2,189.21) .. controls (407.52,188.79) and (410,189.66) .. (410.74,191.16) .. controls (411.47,192.65) and (410.18,194.2) .. (407.86,194.62) .. controls (405.53,195.04) and (403.05,194.16) .. (402.32,192.67) -- cycle ;
\draw    (403.03,193.86) -- (382.01,207.72) ;
\draw    (373.52,188.1) -- (401.7,190.73) ;
\draw [color={rgb, 255:red, 2; green, 122; blue, 250 }  ,draw opacity=0.4 ]   (412.7,190.79) -- (425.49,188.43) -- (427.5,188.05) -- (450.13,183.86) ;
\draw [shift={(410.74,191.16)}, rotate = 349.51] [fill={rgb, 255:red, 2; green, 122; blue, 250 }  ,fill opacity=0.4 ][line width=0.08]  [draw opacity=0] (7.2,-1.8) -- (0,0) -- (7.2,1.8) -- cycle    ;
\draw   (365.11,189.61) .. controls (364.37,188.12) and (365.66,186.57) .. (367.99,186.15) .. controls (370.31,185.73) and (372.79,186.6) .. (373.52,188.1) .. controls (374.26,189.59) and (372.97,191.14) .. (370.64,191.56) .. controls (368.32,191.98) and (365.84,191.11) .. (365.11,189.61) -- cycle ;
\draw   (374.69,209.94) .. controls (373.95,208.44) and (375.24,206.89) .. (377.57,206.47) .. controls (379.89,206.06) and (382.37,206.93) .. (383.1,208.42) .. controls (383.84,209.92) and (382.55,211.47) .. (380.22,211.89) .. controls (377.9,212.3) and (375.42,211.43) .. (374.69,209.94) -- cycle ;
\draw  [color={rgb, 255:red, 155; green, 155; blue, 155 }  ,draw opacity=1 ][fill={rgb, 255:red, 255; green, 255; blue, 255 }  ,fill opacity=1 ] (398.29,205.1) -- (524.5,182.8) -- (523.15,235.57) -- (396.93,257.87) -- cycle ;
\draw  [fill={rgb, 255:red, 74; green, 144; blue, 226 }  ,fill opacity=1 ] (495,213.9) .. controls (494.27,212.4) and (495.56,210.85) .. (497.88,210.44) .. controls (500.21,210.02) and (502.69,210.89) .. (503.42,212.39) .. controls (504.15,213.88) and (502.86,215.43) .. (500.54,215.85) .. controls (498.22,216.27) and (495.74,215.39) .. (495,213.9) -- cycle ;
\draw   (447.32,222.71) .. controls (446.59,221.22) and (447.87,219.66) .. (450.2,219.25) .. controls (452.52,218.83) and (455,219.7) .. (455.74,221.2) .. controls (456.47,222.69) and (455.18,224.24) .. (452.86,224.66) .. controls (450.53,225.08) and (448.05,224.2) .. (447.32,222.71) -- cycle ;
\draw    (448.03,223.9) -- (427.01,237.76) ;
\draw    (418.52,218.14) -- (446.7,220.77) ;
\draw [color={rgb, 255:red, 2; green, 122; blue, 250 }  ,draw opacity=0.4 ]   (457.7,220.83) -- (470.49,218.47) -- (472.5,218.09) -- (495.13,213.91) ;
\draw [shift={(455.74,221.2)}, rotate = 349.51] [fill={rgb, 255:red, 2; green, 122; blue, 250 }  ,fill opacity=0.4 ][line width=0.08]  [draw opacity=0] (7.2,-1.8) -- (0,0) -- (7.2,1.8) -- cycle    ;
\draw   (410.11,219.65) .. controls (409.37,218.16) and (410.66,216.61) .. (412.99,216.19) .. controls (415.31,215.77) and (417.79,216.64) .. (418.52,218.14) .. controls (419.26,219.63) and (417.97,221.18) .. (415.64,221.6) .. controls (413.32,222.02) and (410.84,221.15) .. (410.11,219.65) -- cycle ;
\draw   (419.69,239.98) .. controls (418.95,238.48) and (420.24,236.93) .. (422.57,236.51) .. controls (424.89,236.1) and (427.37,236.97) .. (428.1,238.46) .. controls (428.84,239.96) and (427.55,241.51) .. (425.22,241.93) .. controls (422.9,242.34) and (420.42,241.47) .. (419.69,239.98) -- cycle ;
\draw [color={rgb, 255:red, 4; green, 116; blue, 248 }  ,draw opacity=1 ]   (293.02,73.39) .. controls (497.47,7.08) and (513.84,135.95) .. (501.69,209.64) ;
\draw [shift={(501.5,210.75)}, rotate = 279.65] [fill={rgb, 255:red, 4; green, 116; blue, 248 }  ,fill opacity=1 ][line width=0.08]  [draw opacity=0] (7.2,-1.8) -- (0,0) -- (7.2,1.8) -- cycle    ;
\draw [color={rgb, 255:red, 4; green, 116; blue, 248 }  ,draw opacity=0.4 ]   (458.42,182.35) .. controls (494.12,191.28) and (495.16,197.09) .. (497.5,208.57) ;
\draw [shift={(497.88,210.44)}, rotate = 257.98] [fill={rgb, 255:red, 4; green, 116; blue, 248 }  ,fill opacity=0.4 ][line width=0.08]  [draw opacity=0] (7.2,-1.8) -- (0,0) -- (7.2,1.8) -- cycle    ;
\draw [color={rgb, 255:red, 4; green, 116; blue, 248 }  ,draw opacity=1 ]   (334.58,102.64) .. controls (474.62,75.43) and (504.7,118.32) .. (499.96,209.06) ;
\draw [shift={(499.88,210.44)}, rotate = 273.19] [fill={rgb, 255:red, 4; green, 116; blue, 248 }  ,fill opacity=1 ][line width=0.08]  [draw opacity=0] (7.2,-1.8) -- (0,0) -- (7.2,1.8) -- cycle    ;
\draw [color={rgb, 255:red, 4; green, 116; blue, 248 }  ,draw opacity=1 ]   (383.58,130.64) .. controls (492.9,114.91) and (495.95,152.49) .. (497.83,208.73) ;
\draw [shift={(497.88,210.44)}, rotate = 268.11] [fill={rgb, 255:red, 4; green, 116; blue, 248 }  ,fill opacity=1 ][line width=0.08]  [draw opacity=0] (7.2,-1.8) -- (0,0) -- (7.2,1.8) -- cycle    ;
\draw [color={rgb, 255:red, 4; green, 116; blue, 248 }  ,draw opacity=1 ]   (290.73,75.78) .. controls (349.79,58.6) and (386.51,101.48) .. (383.25,131.95) ;
\draw [shift={(383,133.8)}, rotate = 279.29] [fill={rgb, 255:red, 4; green, 116; blue, 248 }  ,fill opacity=1 ][line width=0.08]  [draw opacity=0] (7.2,-1.8) -- (0,0) -- (7.2,1.8) -- cycle    ;

\draw (184,220) node [anchor=north west][inner sep=0.75pt]  [color={rgb, 255:red, 0; green, 0; blue, 0 }  ,opacity=1 ] [align=left] {time discretization};
\draw (190.77,103.94) node [anchor=north west][inner sep=0.75pt]  [font=\tiny]  {$\bX(0)$};
\draw (234.37,134.14) node [anchor=north west][inner sep=0.75pt]  [font=\tiny]  {$\bX(t_1)$};
\draw (280.37,161.94) node [anchor=north west][inner sep=0.75pt]  [font=\tiny]  {$\bX(t_2)$};
\draw (353.67,211.9) node [anchor=north west][inner sep=0.75pt]  [font=\tiny]  {$\bX(t_{n-1})$};
\draw (398.67,242.94) node [anchor=north west][inner sep=0.75pt]  [font=\tiny]  {$\bX(t_{n})$};

\end{tikzpicture}

}
    \caption{ Model discretization in FROND. Unlike the Euler discretization in graph neural ODE models, FROND incorporates connections to historical times, introducing memory effects. Specifically, the dark blue connections observed in FROND at $\beta<1$ are absent in ODEs (corresponding to $\beta=1$). The weight of these skip connections correlates with $b_{j,k+1}(\beta)$ as detailed in \cref{eq.bjk}.}
    \label{fig.block}
\end{figure}
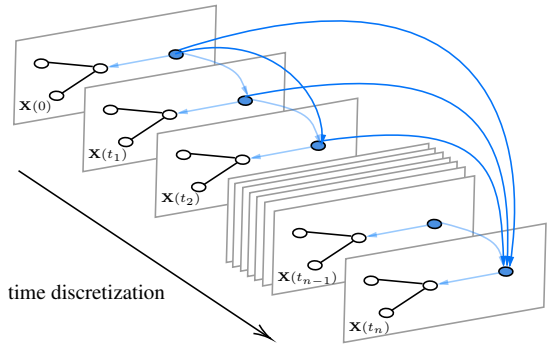
\section{Related Work}

\subsection{Graph Neural ODE Models}
The interplay between ODEs and neural networks has recently shed light on the potential of continuous dynamical systems in deep learning frameworks.
This concept is initially explored in the work of \cite{weinan2017proposal}. A landmark study by \cite{chen2018neural} further cements this notion, presenting neural ODEs equipped with open-source solvers. This methodology allows for a more precise alignment of the inputs and outputs of neural networks with established physical principles, thereby increasing the networks' interpretability. This field continues to evolve, with notable advancements in enhancing neural network efficiency \cite{dupont2019augmented}, bolstering robustness \cite{yan2019robustness,kang2021Neurips}, and stabilizing gradient functions \cite{haber2017stable}. In parallel, \cite{avelar2019discrete,poli2019graph} demonstrate the utilization of continuous residual GNN layers, leveraging neural ODE solvers to optimize output. Recently, GraphCON \cite{rusch2022icml} has implemented the coupled oscillator model, which effectively maintains the Dirichlet energy of graphs over time, addressing the prevalent issue of over-smoothing in these networks. The diffusion theory, as conceptualized in \cite{chamberlain2021grand}, likens information propagation to the diffusion process of substances. The Beltrami diffusion models, as utilized in \cite{charoweyn:blend2021,SonKanWan:C22}, have been pivotal in improving the rewiring and robustness of graphs. Concurrently, ACMP \cite{wang2022acmp} draws inspiration from particle reaction-diffusion processes, accounting for both repulsive and attractive interactions among particles. The graph CDE model, as outlined in \cite{zhao2023cde}, addresses heterophilic graph challenges, inspired by convection-diffusion processes. Similarly, GREAD \cite{choi2023gread} introduces an approach based on reaction-diffusion equations, tailored to effectively manage heterophilic datasets.
GRAND++ \cite{thorpe2021grand++} employs heat diffusion with sources for training models more efficiently, especially when there is a scarcity of labeled data. Further enriching this field, recent research \cite{ZhaKanSon:C23b,KanZhaSon:C23} adopts the Hamiltonian mechanism for updating node features, thereby augmenting the networks' adaptability to graph structures and enhancing their robustness.
\subsection{Adversarial Attacks and Defenses on Graphs}
A plethora of research has consistently underscored the vulnerability of graph deep learning models to adversarial perturbations. Essentially, even inconspicuous alterations to the input data can misdirect a graph neural network into producing fallacious predictions. Adversarial attacks on GNNs typically fall into two categories based on the method of perturbation: Graph Modification Attacks (GMA) and Graph Injection Attacks (GIA).
GMA involves manipulating the topology of a graph, primarily by adding or removing edges \cite{Chen2018FastGA,WaniekNHB2018,Du2017TopologyAG,maKDD2021,softmehiangcn}. This category also encompasses perturbations to node features \cite{zugnernettack,zugner_adversarial_2019,maKDD2021, ma2022adversarial,finkelshtein2022single}.
In contrast, Graph Injection Attacks (GIA) permit adversaries to incorporate malicious nodes into the original graph \cite{Wang2020ScalableAO,zouKDD2021,sunadvinject,hussain2022adversarial,chen2022hao}. GIA is considered a stronger form of attack on graph data \cite{chen2022hao} as it introduces both structural and feature perturbations to the graph.

The defensive strategies employed in GNNs can be broadly categorized into pre-processing methods and the design of robust architectures.
Methods such as GNN-GUARD \cite{zhang2020gnnguard}, Pro-GNN \cite{jin2020prognn},GARNET\cite{deng2022garnet} and GCN-SVD \cite{Entezarigcnsvd} focus on cleansing or pruning the graph, with the aim of maintaining the integrity of the original adjacency matrix, thereby mitigating perturbations. On another front, methods like RGCN \cite{zhu2019robustgcn} and Soft-Median-GCN \cite{softmehiangcn} are tailored to enhance the inherent architecture of GNNs, making them more resilient to feature perturbations. Distinctly, our approach diverges from these conventional defense mechanisms. \emph{Instead of proposing an entirely new defensive technique, our focus is on bolstering the robustness of existing graph neural ODE models by seamlessly integrating the principles of FDEs.}

\section{Preliminaries}
\subsection{Notation}
Let us consider a graph $\calG = (\calV, \bW)$, in which $\calV = \{1, \dots, N\}$ represents a set of $N$ nodes. The $N \times N$ matrix $\bW := \left(W_{ij}\right)$ has elements $W_{ij}$ indicating the original edge weight between the $i$-th and $j$-th nodes with $W_{ij} = W_{ji}$.
The node features at any given time $t$ can be denoted by $\bX(t) \in \Real^{|\calV|\times N}$, where $N$ corresponds to the dimension of the node feature. 
In this matrix, the feature vector for the $i$-th node in $\calV$ at time $t$ can be represented as the $i$-th row of $\bX(t)$, indicated by $\bx\T_i(t)$.

\subsection{Graph Neural ODE Models} 
Existing research encompasses various continuous dynamics-informed GNNs, with unique configurations of $\calF$ in \cref{eq.graphode} tailored for graph dynamics. This section provides a succinct overview of several graph neural ODE models that we will employ in this study. For an extensive review of these and related GNNs, readers are referred to a recent comprehensive survey \cite{han2023continuous}.

GRAND \cite{chamberlain2021grand} incorporates the following dynamical system for graph learning:
\begin{align}
\frac{{\ud} \bX(t)}{{\ud} t}
&= \operatorname{div}(D(\bX(t), t) \odot \nabla \bX(t))\nn
& = (\mathbf{A}(\bX(t))-\mathbf{I}) \bX(t) \label{eq.GRAND}
\end{align}
The initial condition $\bX(0)$ is provided by the graph input features. Here, $\odot$ represents the element-wise product, and $D$ is a diagonal matrix with elements $\diag(a(\bx_i(t), \bx_j(t)))$. The function $a(\cdot)$ serves as a measure of similarity for node pairs $(i,j)$ linked by an edge, that is, when $W_{ij}\ne 0$. 
As such, the diffusion equation can be reframed as \cref{eq.GRAND}, where $\mathbf{A}(\bX(t))= \left(a\left(\bx_i(t), \bx_j(t)\right)\right)$ constitutes a learnable attention matrix to depict the graph structure. $\mathbf{I}$ is the identity matrix.
One way to calculate $a(\bx_i,\bx_j)$ is based on the Transformer attention \cite{vaswani2017attention}:
\begin{align}
    a(\bx_i,\bx_j) = \mathrm{softmax} \left(\frac{(\mathbf{W}_K \bx_{i} )^\top \mathbf{W}_Q \bx_{j}}{d_k}\right)
\end{align}
where $\mathbf{W}_K$ and $\mathbf{W}_Q$ are learned matrices, and $d_k$ is a hyperparameter determining the dimension of $\mathbf{W}_K$. 

By extending the concepts of Beltrami flow \cite{sochenTIP1998beltrami,SonKanWan:C22}, a stable graph neural flow GraphBel is formulated as:
 \begin{align} \label{eq:GraphBel}
\frac{{\rm d} \mathbf{X}(t)}{{\rm d} t} = (\mathbf{A_S}(\mathbf{X}(t)) \odot \mathbf{B_S}(\mathbf{X}(t)) - \Psi(\mathbf{X}(t))) \mathbf{X}(t)
\end{align}
where $\odot$ represents element-wise multiplication. Both $\mathbf{A_S}(\cdot)$ and $\mathbf{B_S}(\cdot)$ serve distinct purposes: the former acts as a learnable attention function, while the latter operates as a normalized vector map. $\mathbf{\Psi}(\mathbf{X}(t))$ is a diagonal matrix where $\Psi(\mathbf{x}_i, \mathbf{x}_i)=\sum_{\mathbf{x}_j}(\mathbf{A_S} \odot \mathbf{B_S})(\mathbf{x}_i, \mathbf{x}_j)$. 

Using a graph coupled dynamical system, GraphCON \cite{ruscharow:graphcon2022} is given by
\begin{align}
\begin{aligned}
& \frac{{\ud} \mathbf{Y}(t)}{{\ud} t} = \sigma ({\mathbf{F}_{\theta}(\mathbf{X}(t), t)) - \gamma\mathbf{X}(t) -\alpha\mathbf{Y}(t)} \\
& \frac{{\ud} \mathbf{X}(t)}{{\ud} t} = \mathbf{Y}(t) \label{eq.graphcon}
\end{aligned}
\end{align}
where $\mathbf{F}_{\theta}(\cdot)$ is a learnable $1$-neighborhood coupling function, $\sigma$ denotes an activation function, $\gamma$ and $\alpha$ are adjustable parameters.   

\begin{Remark}
By leveraging the numerical solvers introduced in \cite{chen2018neural}, one can efficiently solve \cref{eq.GRAND}, \cref{eq:GraphBel}, and \cref{eq.graphcon} where the initial $\bX(0)$ represents the input features. This yields the terminal node embeddings, denoted as $\mathbf{X}(T)$, at time $T$. Subsequently, $\mathbf{X}(T)$ can be utilized for downstream tasks such as node classification or link prediction.
\end{Remark}

\subsection{Fractional-Order Differential Equation}\label{ssec.fde}
Within the FROND framework, the fractional time derivative is typically characterized using the Caputo derivative \cite{caputo1967fracderiva}, a prevalent choice for modeling real-world phenomena \cite{diethelm2010frationalde}. It is expressed as:
\begin{align}
    D_t^\beta f(t) = \frac{1}{\Gamma(n-\beta)} \int_0^t (t-\tau)^{n-\beta-1} \frac{\ud^n f}{\ud \tau^n} \ud\tau, \label{eq.cap}
\end{align}
where $\beta$ is the fractional order, $n$ is the smallest integer greater than $\beta$, $\Gamma$ is the gamma function, $f$ is a scalar function defined over some interval that includes $[0,t]$, and $\frac{\ud^n f}{\ud \tau^n}$ is the standard $n$-th order derivative. 
A distinguishing trait of the Caputo derivative is its capability to incorporate \emph{memory effects}. This is underscored by observing that the fractional derivative at time $t$ in \cref{eq.cap} \emph{aggregates historical states spanning the interval $0\le \tau\le t$.}
For the special case where $\beta=1$, the definition collapses to the standard first-order derivative as $D_t^\beta f = \frac{\ud f}{\ud \tau}$. 
For a vector-valued function, the fractional derivative is defined component-wise for each dimension, similar to the integer-order derivative. Thus, while our discussion centers on scalar functions in \cref{ssec.fde,ssec.solver}, its extension to vector-valued functions is straightforward.
A more detailed, self-contained exposition of the Caputo derivative can be found in the supplementary material.

A crucial concept in fractional calculus and its applications is the Mittag-Leffler function $E_{\beta}(z)$ \cite{diethelm2010frationalde}. Recognized as a natural extension of the exponential function within fractional domains, it enables the modeling of complex phenomena with increased sophistication. 
The Mittag-Leffler function is a crucial component in the solutions to numerous fractional differential equations, thus playing an essential role in the analysis and application of such systems.
Specifically, as per \cite{diethelm2010frationalde}[Theorem 4.3], given $y(t):=E_n\left(\lambda t^\beta\right)$, $x \geq 0$, then 
\begin{align}
D_t^\beta y(t)=\lambda y(t). \label{eq.eigen_frac}
\end{align}
We present the formal definition of the Mittag-Leffler function below:
\begin{Definition}[Mittag-Leffler function]\label{def.ml}
Let $\beta > 0$. The function $E_{\beta}$ defined by
\begin{align}
E_{\beta}(z) \coloneqq \sum_{j=0}^{\infty} \frac{z^j}{\Gamma (j{\beta}+1)},
\end{align}
whenever the series converges, is called the Mittag-Leffler function of order $\beta$.
\end{Definition}
The extension of the Mittag-Leffler function from the exponential function is apparent when considering the case $\beta=1$, which simplifies to the well-known exponential function:
\begin{align}
E_1(z) = \sum_{j=0}^{\infty} \frac{z^j}{\Gamma (j+1)} = \sum_{j=0}^{\infty} \frac{z^j}{j!} = \exp(z).
\end{align}
It is also well-recognized that $\exp (z)$ acts as the eigenfunction for ODEs. Specifically, $\exp(\lambda t)$ solves \cref{eq.eigen_frac} for $\beta=1$, assuming appropriate initial conditions are met.
\subsection{Numerical Solvers for FDEs}\label{ssec.solver}
In the context of discrete numerical solvers, FDEs can be solved analogously to ODEs as illustrated in \cite{chen2018neural}. Particularly noteworthy is the fractional Adams–Bashforth–Moulton method solver, which shares similarities with the Adams–Bashforth–Moulton technique for ODEs, as expounded in \cite{diethelm2004solvers}.

For clarity, consider an FDE characterized by $\beta\in (0,1]$:
\begin{align}
       D_t^\beta y(t) = f(t, y(t)), \quad y(0) = y_0.
\end{align}
Here, $f(t, y(t))$ delineates the dynamics of the system, and $y_0$ specifies the initial condition at $t=0$.

Delving into numerical approximations and drawing upon \cite{diethelm2004solvers}, the basic predictor solution $y^P_{k+1}$ (where $^P$ signifies the concept of the "Predictor") derives from the fractional Adams--Bashforth method as:
\begin{align}
  y^P(t_{k+1}) = y_0 + \frac{1}{\Gamma(\beta)} \sum_{j=0}^{k} b_{j,k+1}(\beta) f(t_j, y_j).   \label{eq.bjk}
\end{align}
In this context, $k$ stands for the current iteration or time step in the discretization sequence. Further, with $h$ denoting the step size or time interval between subsequent approximations, $t_j$ is given by $t_j=hj$. The coefficients $b_{j,k+1}(\beta)$, expressed as functions of $\beta$, are elaborated in the supplementary material. The reader is directed to \cref{fig.block}, where the role of $b_{j,k+1}(\beta)$ as a weighted skip connection in time discretization, underscoring memory effects, is evident.

\section{Methodology} \label{sec.metho}
In this section, we present the theoretical analysis of the output boundary of \cref{eq.frac_gra_difaaa} under specific perturbation scenarios. By leveraging the characteristics of the Mittag-Leffler function, we detail the response of the FROND, noting that it undergoes smaller output perturbations compared to the graph neural ODE framework when exposed to the same disturbances. We furnish three pivotal theorems underscoring the inherent resilience of the FROND paradigm:
\begin{itemize}
    \item Theorem \ref{theo:feature_pert} \cite{diethelm2010frationalde} establishes the output perturbation bounds of the FDEs under small perturbations in the initial conditions, which, in our case, correspond to input feature changes of FROND models. 
    \item Theorem \ref{theo:function_pert} \cite{diethelm2010frationalde} extends the discussion to include perturbations in the function that governs the system's dynamics. In the context of graph learning, such perturbations include the changes in the topology of the graph, which can occur due to the addition, deletion, or modification of edges. 
    \item Finally, Theorem \ref{theo:ml_mono} provides important insights into how the choice of the fractional order $\beta$ can influence the system's robustness. Our analysis suggests a monotonic relationship between the model's perturbation bounds and the parameter $\beta$, with smaller $\beta$ values indicating augmented robustness.
\end{itemize}
 Together, these results provide a strong theoretical basis for the robustness of FROND, setting the stage for its deployment in various practical applications.

\subsection{FROND: Graph Neural FDE Framework}
Building upon the foundation of FDEs, recall that FROND incorporates the Caputo fractional-order time derivative into the model for feature evolution:
\begin{align}
D_t^\beta \bX(t)= \calF(\bW,\bX(t)),\
\bX(0) = \bX_{0}, \ 0<\beta\le 1. \label{eq.frac_gra_difaaa}
\end{align}
In equation \eqref{eq.frac_gra_difaaa}, $D_t^\beta \bX(t)$ represents the fractional derivative of the state $\bX(t)$ with respect to feature evolution time $t$, where $\beta$ is a real number in the interval $(0, 1]$. This fractional derivative introduces memory effects, which enrich the capacity of the model to interpret complex patterns in data. The term $\calF(\bW,\bX(t))$ denotes the dynamic function modeling the interactions between nodes given graph topology $\bW$. 
With the system initialized from input node features as $\bX(0) = \bX_{0}$, the model's output is $\bX(T)$ at a specified time $T$.

One intrinsic characteristic of FROND is the long-memory property from fractional derivative. This property encapsulates the system's ability to ``remember'' its historical states. This inherent memory effect contributes significantly to the robustness of the system, particularly when faced with perturbations.
When the system encounters disturbances or noise, the extensive memory of FROND serves as a protective buffer. It mitigates the immediate effects of these disruptions by integrating the system's past states into its response, rather than amplifying the disturbances.

\subsection{Robustness of FROND under Perturbation}
In this subsection, we delve into a theoretical analysis of the model output perturbations. We begin by highlighting two central theorems from \cite{diethelm2010frationalde}, which provide bounds on output perturbations, grounded in the properties of the Mittag-Leffler function. Following this, we analyze the bound outlined in \cref{theo:ml_mono}, shedding light on the notion that a smaller $\beta$ contributes to enhanced robustness of the model, particularly when faced with perturbations in input node features and graph topology.
\begin{figure}
    \centering
    \includegraphics[width=0.45\textwidth]{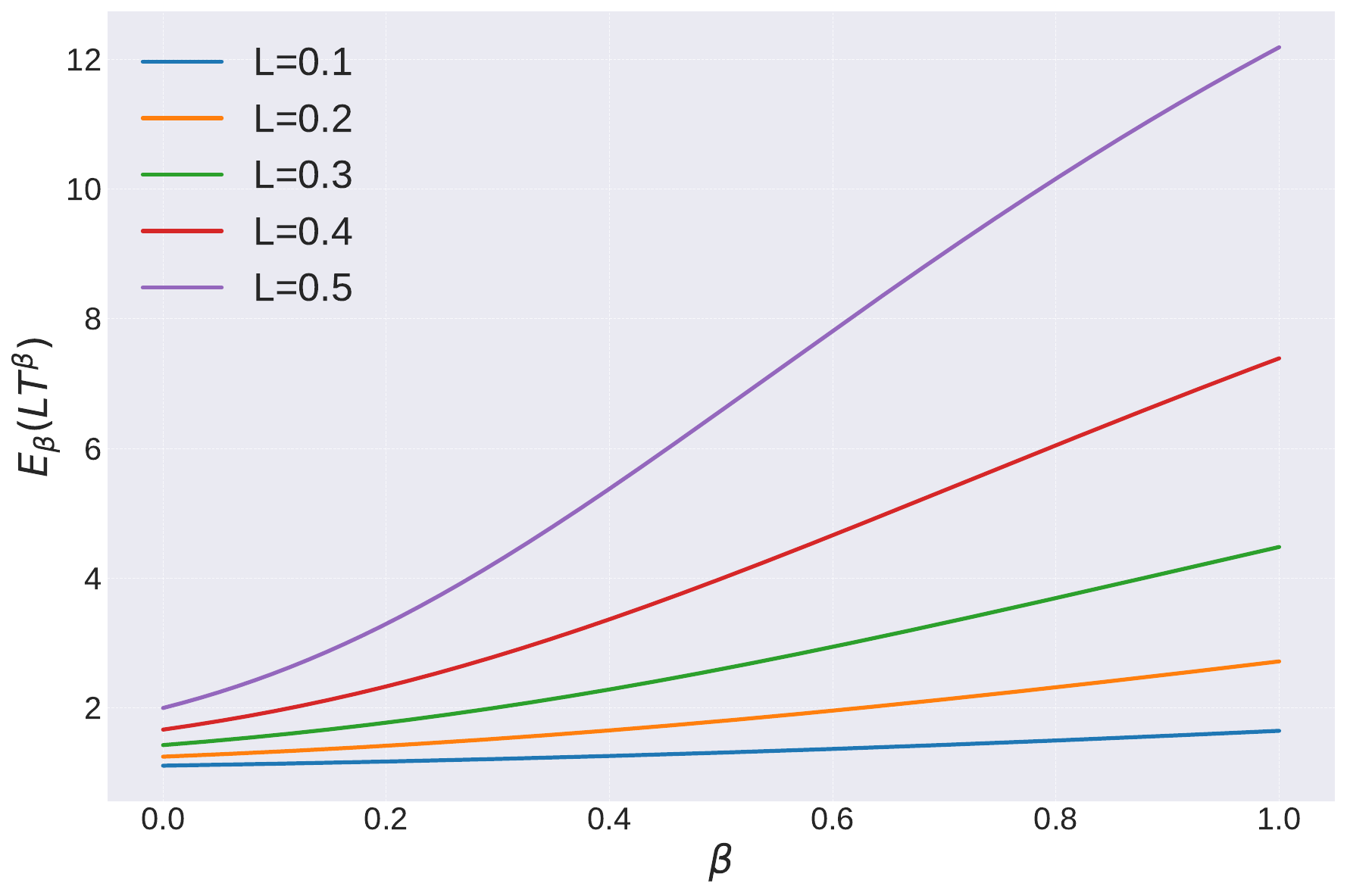}
    \caption{Plot of the Mittag-Leffler function $E_{\beta}(L T^{\beta})$ against $\beta$ with $T=10$. Distinctively, for varying $L$, it displays  monotonic increase over interval $[\epsilon,1]$. 
    }
    \label{fig:enter-label}
\end{figure}
\begin{Theorem}\label{theo:feature_pert} \cite[Theorem 6.20]{diethelm2010frationalde} 
Let $\bX(t)$  be the solution of the initial value problem \cref{eq.frac_gra_difaaa}, and let $\Tilde{\bX}(t)$ be the solution of the initial value problem
\begin{align}
\begin{aligned}
D_t^\beta \Tilde{\bX}(t) &= \calF(\bW,\Tilde{\bX}(t)),  \\
\Tilde{\bX}(0) &=  \Tilde{\bX}_{0},
\end{aligned}
\end{align}
where \(\varepsilon := \|\bX_{0} - \Tilde{\bX}_{0}\|\). Then, if \(\varepsilon\) is sufficiently small, there exists some \(h > 0\) such that both the functions \(\bX\) and \(\Tilde{\bX}\) are defined on \([0, h]\), and
\begin{align}
  \sup_{0\leq t\leq h} \|\bX(t)-\Tilde{\bX}(t)\| = c_1\varepsilon E_{\beta}(Lh^{\beta})  
\end{align}
where ${L}$ is the Lipschitz constant of $\calF$ and $c_1$ is a constant.
\end{Theorem}
\begin{Remark}
   In our FROND framework, \cref{theo:feature_pert} provides an upper bound for the perturbation in system trajectory. This encompasses the perturbation of the FROND output $\bX(T)$ at time $T$ if $T<h$, in situations of slight perturbations to the input features, specifically when the system's initial input features shift to $\Tilde{\bX}_{0}$.
\end{Remark}

\begin{Theorem} \cite[Theorem 6.21]{diethelm2010frationalde} \label{theo:function_pert}
Let $\bX(t)$  be the solution of the initial value problem \cref{eq.frac_gra_difaaa}, and let $\Tilde{\bX}(t)$ be the solution of the initial value problem
\begin{equation}
\begin{aligned}
D_t^\beta \Tilde{\bX}(t) &= \Tilde\calF(\Tilde{\bW},{\bX}(t)),  \\
\Tilde{\bX}(0) &=  {\bX}_{0}
\end{aligned}
\end{equation}
Moreover, let \(\varepsilon \coloneqq \sup_{\bX\in A} \|\calF(\bW,{\bX}) - \Tilde\calF(\Tilde{\bW},{\bX})\|\), with $A$ being an appropriate compact set where solutions for both systems exist. Then, if $\varepsilon$ is sufficiently small, there exists some \(h > 0\) such that both the functions \(\bX\) and \(\Tilde{\bX}\) are defined on \([0, h]\), and
\begin{align}
  \sup_{0\leq t\leq h} \|\bX(t)-\Tilde{\bX}(t)\| = c_2\varepsilon E_{\beta}(Lh^{\beta}), 
\end{align}
where ${L}$ is the Lipschitz constant of $\calF$ and $c_2$ is a constant.
\end{Theorem}
\begin{Remark}
   Within the framework of FROND, the discussion relates to how the model's output reacts to perturbations in functional elements (such as learned parameters in $\calF$) and changes in graph structure $\bW$.
    In our paper, we consider topological changes in the graph structure, such as edge additions, deletions, or modifications.
   Further investigation into attacks related to functional perturbations that directly alter the parameters of the neural network $\calF$ is earmarked for future work.
\end{Remark}

\begin{table*}[!ht]
\centering
\fontsize{9pt}{10pt}\selectfont
\setlength{\tabcolsep}{5pt}
\makebox[\textwidth][c]{
\begin{tabular}{ccccccccccc} 
\toprule
Dataset & Ptb(\%) & F-GRAND & GRAND & F-GraphBel  & GraphBel & F-GraphCON  & GraphCON     & GAT & GCN  \\

\midrule
\multirow{6}{*}{Cora}

& 0 &  81.25$\pm$0.89 &  82.24$\pm$1.82 &  79.05$\pm$0.73 & 80.28$\pm$0.87  & 80.91$\pm$0.54 & 83.10$\pm$0.79  & \first{83.97$\pm$0.65} & \second{83.50$\pm$0.44}     \\

& 5 &  78.84$\pm$0.57 & \second{78.97$\pm$0.49} & 76.10$\pm$0.74  & 77.70$\pm$0.66 &  77.80$\pm$0.44 &  77.90$\pm$1.14  & \first{80.44$\pm$0.74} & 76.55$\pm$0.79   \\

& 10 & \first{76.61$\pm$0.68}& \second{75.02$\pm$1.25} & 74.03$\pm$0.47  & 74.30$\pm$0.88 &  74.63$\pm$1.42 & 72.53$\pm$1.08 & 70.39$\pm$1.28 & 70.39$\pm$1.28     \\

& 15&  \first{73.42$\pm$0.97} & 71.43$\pm$1.09 & \second{73.01$\pm$0.75}  & 72.14$\pm$0.69 & 73.01$\pm$0.78 & 69.83$\pm$0.68 & 65.10$\pm$0.71 & 65.10$\pm$0.71     \\

& 20 &  \second{69.27$\pm$2.10} & 60.53$\pm$1.99 & \first{69.35$\pm$1.23}  & 65.41$\pm$0.99 & 69.23$\pm$1.35 & 57.28$\pm$1.62 & 59.56$\pm$2.72 & 59.56$\pm$2.72     \\

& 25 &  64.47$\pm$1.83 & 55.26$\pm$2.14 & \first{67.63$\pm$0.93} & 62.31$\pm$1.13 &  \second{65.27$\pm$1.33} & 53.17$\pm$1.52 & 47.53$\pm$1.96 & 47.53$\pm$1.96   \\
\midrule
\multirow{6}{*}{Citeseer}

& 0 &  71.37$\pm$1.34 &  71.50$\pm$1.10 & 68.90$\pm$1.15  & 69.46$\pm$1.15 & 71.49$\pm$0.71 & 70.48$\pm$1.18 & \first{73.26$\pm$0.83} & \second{71.96$\pm$0.55}    \\

& 5 &  \second{71.47$\pm$0.96} &  71.04$\pm$1.15 &  68.36$\pm$0.93  & 68.45$\pm$1.02 & 70.77$\pm$1.15 & 69.75$\pm$1.63 &  \first{72.89$\pm$0.83} & 70.88$\pm$0.62    \\

& 10 &  \second{69.76$\pm$0.71} & 68.88$\pm$0.60  & 67.22$\pm$1.52  & 66.72$\pm$1.31 & 69.54$\pm$0.82 & 67.40$\pm$1.78 & \first{70.63$\pm$0.48} & 67.55$\pm$0.89   \\

& 15&  \second{67.94$\pm$1.42} &  66.35$\pm$1.37  &  63.56$\pm$1.95  & 63.63$\pm$1.67 & 67.37$\pm$0.87 &  65.78$\pm$1.97 & \first{69.02$\pm$1.09} & 64.52$\pm$1.11     \\

& 20 &  \second{64.18$\pm$0.93} & 58.71$\pm$1.42 & 63.38$\pm$0.96  & 58.90$\pm$0.84 & \first{66.52$\pm$0.68} & 56.79$\pm$1.46  & 61.04$\pm$1.52 & 62.03$\pm$3.49   \\

& 25 &  \second{65.46$\pm$1.12} & 60.15$\pm$1.37 & 64.60$\pm$0.48 &  61.24$\pm$1.28 & \first{66.72$\pm$1.12} & 57.30$\pm$1.38  & 61.85$\pm$1.12 & 56.94$\pm$2.09    \\

\midrule
\multirow{7}{*}{Pubmed}

& 0 &  \first{87.28$\pm$0.23} & 85.06$\pm$0.26 & 86.34$\pm$0.15  & 84.02$\pm$0.26 &  87.12$\pm$0.21 &  84.65$\pm$0.13 & 83.73$\pm$0.40 & \second{87.19$\pm$0.09}  \\

& 5 & \first{87.05$\pm$0.17} & 84.11$\pm$0.30 & \second{86.17$\pm$0.12} & 83.91$\pm$0.26 &  86.72$\pm$0.23 & 83.06$\pm$0.22  & 78.00$\pm$0.44 & 83.09$\pm$0.13   \\

& 10 &  \first{86.74$\pm$0.23} & 84.24$\pm$0.18 & 86.01$\pm$0.18 & 84.62$\pm$0.26 & \second{86.64$\pm$0.20} & 82.25$\pm$0.12  & 74.93$\pm$0.38 & 81.21$\pm$0.09    \\

& 15 &  \second{86.51$\pm$0.14} & 83.74$\pm$0.34 & 85.92$\pm$0.13  & 84.83$\pm$0.20 &  \first{86.40$\pm$0.14} & 81.26$\pm$0.33 & 71.13$\pm$0.51 & 78.66$\pm$0.12   \\

& 20 &  \first{86.50$\pm$0.12} & 83.58$\pm$0.20 & 85.73$\pm$0.18 & 84.89$\pm$0.45 & \second{86.32$\pm$0.12} &  81.58$\pm$0.41 & 68.21$\pm$0.96 & 77.35$\pm$0.19  \\

& 25 &  \first{86.47$\pm$0.15} & 83.66$\pm$0.25 & 86.11$\pm$0.30  & 85.07$\pm$0.15 &  \second{86.15$\pm$0.26} & 80.75$\pm$0.32  & 65.41$\pm$0.77 & 75.50$\pm$0.17  \\

\bottomrule
\end{tabular}} 
\caption{Node classification accuracy (\%) under { modification, poisoning, non-targeted}  attack (Metattack) in {transductive} learning.  The best and the second-best results for each criterion are highlighted in bold and underlined, respectively. }
 \label{tab:metattack} 
\end{table*}

\begin{Theorem}\label{theo:ml_mono}
Let $f(\beta)=E_{\beta}(L T^{\beta})$. For any $\epsilon>0$, if $T$ is sufficiently large and $L<1$,  $f(\beta)$ is monotonically increasing on the interval $[\epsilon,1]$. 
\end{Theorem}
\begin{proof}
     See the supplementary material for the proof. 
\end{proof}
\begin{Remark} 
In conjunction with \cref{theo:feature_pert,theo:function_pert}, \cref{theo:ml_mono} shows that the fractional order $\beta$ of the FROND plays a crucial role in the model's robustness. With an appropriately chosen $\beta$, the model can reduce the discrepancy between the clean and perturbed states, thereby improving the robustness. 
Particularly, a smaller $\beta$ is associated with a smaller discrepancy, signifying enhanced robustness of FROND against perturbations when $\beta<1$ compared to graph neural ODE models with $\beta=1$. 
Please refer to \cref{fig:enter-label,fig:beta_model} for an illustration. 
The monotonicity suggests that a larger $\beta$ results in a larger perturbation bound for the FROND solution at time $T$, thus expecting a larger perturbed output at the same time under identical input/graph topology perturbations.
\end{Remark}

\subsection{Algorithms}
Our proposed approach enhances the robustness of integer-order graph neural ODE models by introducing fractional-order derivatives into the model framework. Specifically, we extend three prominent graph neural ODE models, GRAND, GraphBel, and GraphCON through this method.

We upgrade the GRAND framework with a fractional-order derivative, resulting in the Fractional-GRAND (F-GRAND) model. The F-GRAND formulation is as follows:
\begin{align}
D_t^\beta \bX(t)
= (\mathbf{A}(\bX(t))-\mathbf{I}) \bX(t). \label{eq.Fractional-GRAND}
\end{align}

Following a similar approach, the GraphBel model is modified to incorporate a fractional-order derivative, resulting in the Fractional-GraphBel (F-GraphBel) model. This model is expressed as:
\begin{align} \label{eq:fractional-GraphBel}
D_t^\beta \mathbf{X}(t) = (\mathbf{A_S}(\mathbf{X}(t)) \odot \mathbf{B_S}(\mathbf{X}(t)) - \Psi(\mathbf{X}(t))) \mathbf{X}(t).
\end{align}
Additionally, we introduce the Fractional-GraphCON (F-GraphCON) model, described by the following equations:
\begin{align}
\begin{aligned}
& D_t^\beta \mathbf{Y}(t) = \sigma ({\mathbf{F}_{\theta}(\mathbf{X}(t), t)) - \gamma\mathbf{X}(t) -\alpha\mathbf{Y}(t)}, \\
& D_t^\beta \mathbf{X}(t) = \mathbf{Y}(t). \label{eq.fractional-graphcon}
\end{aligned}
\end{align}
The order $\beta$ of these fractional derivatives serves as a hyperparameter, introducing extra flexibility to these models. This flexibility allows for adaptation to specific data characteristics, enhancing the robustness of the learning process.

\begin{table*}[!htp]
\centering
\fontsize{9pt}{10pt}\selectfont
\setlength{\tabcolsep}{3.3pt}
\makebox[\textwidth][c]{
\begin{tabular}{p{1.2cm}cccccccccccc} 
\toprule
Dataset & Attack & F-GRAND  & GRAND & F-GraphBel & GraphBel & F-GraphCON & GraphCON   & GAT   & GCN  \\
\midrule
\multirow{4}{*}{Cora} 
& \emph{clean} &  \first{86.44$\pm$0.31}  &  85.87$\pm$0.59  & 77.55$\pm$0.79  & 79.07$\pm$0.46  &  82.42$\pm$0.89  & 83.10$\pm$0.63  & \second{86.37$\pm$0.56}  &  85.09$\pm$0.26 \\
& PGD & 56.38$\pm$6.39 &  36.80$\pm$1.86  &  \first{69.50$\pm$2.83}  &  {\second{63.93$\pm$3.88}}  & 56.70$\pm$4.36  & 48.38$\pm$2.44  & 38.82$\pm$2.48 &   40.11$\pm$0.70 \\
& TDGIA &  \second{54.88$\pm$6.72} &  40.0$\pm$3.52  & \first{56.94$\pm$1.82}  & 53.22$\pm$2.95  & 54.24$\pm$2.54  & 46.43$\pm$2.82  & 32.76$\pm$3.30 &  40.43$\pm$1.76 \\
& MetaGIA  & 53.36$\pm$5.31 & 37.89$\pm$1.56  & \first{71.98$\pm$1.32}  &  \second{66.74$\pm$3.23}  & 63.97$\pm$2.09  & 52.21$\pm$2.71  &  42.23$\pm$4.19 &   42.52$\pm$0.90 \\
\midrule

\multirow{4}{*}{Citeseer} 
& \emph{clean}  & 71.91$\pm$0.43 & 72.52$\pm$0.73   &  71.09$\pm$0.30  &  \first{74.75$\pm$0.28}  &  73.50$\pm$0.43  & 72.07$\pm$0.93 &  73.10$\pm$0.39 &  \second{74.48$\pm$0.66}   \\
& PGD   & \first{61.26$\pm$1.23} & 42.20$\pm$2.77   & \second{60.78$\pm$2.37}  & 47.73$\pm$5.87  & 54.47$\pm$1.0  &  37.71$\pm$7.0 &  35.12$\pm$12.44 &   30.49$\pm$0.80   \\
& TDGIA   & 50.74$\pm$1.20 & 30.02$\pm$1.33   &  \first{65.52$\pm$0.55}  & 47.88$\pm$1.83  & \second{54.71$\pm$1.69}  & 30.93$\pm$3.00 &  28.64$\pm$4.05 &  28.88$\pm$2.07   \\
& MetaGIA  & \second{55.50$\pm$1.72} & 30.42$\pm$1.87   &  \first{60.85$\pm$1.88}  & 39.13$\pm$1.19  & 48.82$\pm$3.27  & 29.09$\pm$2.01 &  30.17$\pm$2.71 &  32.74$\pm$1.00   \\
\midrule

\multirow{4}{*}{Computers} 
& \emph{clean}  & \first{92.61$\pm$0.20} &  \second{92.53$\pm$0.34}   & 88.02$\pm$0.24  & 88.12$\pm$0.33 & 91.86$\pm$0.38  & 91.30$\pm$0.20   &  91.42$\pm$0.22 &  91.83$\pm$0.25 \\
& PGD    & \second{89.90$\pm$1.33} &  70.45$\pm$11.03   & 87.60$\pm$0.33  & 87.38$\pm$0.37 & \first{91.36$\pm$0.74}  & 81.28$\pm$7.99   &  38.82$\pm$5.53 &  33.43$\pm$0.21 \\
& TDGIA   &  84.71$\pm$1.52 & 65.45$\pm$14.30   & \second{87.81$\pm$0.28}  & 87.67$\pm$0.40 & \first{90.45$\pm$0.71}  & 68.70$\pm$15.67   &  42.04$\pm$9.01 &  39.83$\pm$3.15 \\
& MetaGIA   & 87.50$\pm$3.17 & 70.01$\pm$9.32   & 87.37$\pm$0.23  & \second{87.77$\pm$0.22} & \first{90.51$\pm$0.88}  & 82.43$\pm$8.42   &  41.86$\pm$8.33 &  34.03$\pm$0.36 \\
\midrule

\multirow{4}{*}{Pubmed} 
& \emph{clean}  & 88.39$\pm$0.47 &  88.44$\pm$0.34   & \second{89.51$\pm$0.12}   & 88.18$\pm$1.89   & \first{90.30$\pm$0.11}  & 88.09$\pm$0.32 &  87.41$\pm$1.73 &  88.46$\pm$0.20 \\
& PGD  & 59.62$\pm$11.66 & 44.61$\pm$2.78  & \first{82.09$\pm$0.83}  & \second{67.81$\pm$12.23}  &  51.16$\pm$6.04  & 45.85$\pm$1.97  &  48.94$\pm$12.99 &  39.03$\pm$0.10 \\
& TDGIA  & 54.31$\pm$2.38 & 46.26$\pm$1.32  &  \first{82.72$\pm$0.47}  & \second{68.66$\pm$10.64}  & 55.50$\pm$4.03  & 45.57$\pm$2.02  & 47.56$\pm$3.11 &  42.64$\pm$1.41 \\
& MetaGIA  & 61.62$\pm$9.05 & 44.07$\pm$2.11  & \first{79.16$\pm$0.87}   & \second{64.64$\pm$9.70}  & 52.03$\pm$5.53  & 45.81$\pm$2.81  & 44.75$\pm$2.53 &  40.42$\pm$0.17 \\

\bottomrule
\end{tabular}}
\caption{Node classification accuracy (\%) on graph { injection, evasion, non-targeted} attack  in { inductive} learning.  The best and the second-best results for each criterion are highlighted in bold and underlined, respectively. }
 \label{tab:adv_trans_1} 
\end{table*}

\section{Experiments}
To empirically validate the robustness of FROND, we carry out a series of experiments where real-world graphs are subjected to various attack methods. The objective of these experiments is to showcase that FROND models, even in the face of such adversarial perturbations, maintain stable performance in downstream tasks, without the need for any additional preprocessing steps to handle the perturbed data. For a comprehensive and fair evaluation, we perform two distinct evaluations: a poisoning Graph Modification Attack (GMA), where training occurs directly on the perturbed graph; and an evasion attack for Graph Injection Attack (GIA), taking place during the inference phase. 

We emphasize that the primary objective of this paper is to investigate the robustness imparted by FROND and to establish that fractional methods exhibit superior robustness compared to GNNs governed by integer-order dynamical systems. Accordingly, our comparisons focus primarily on standard integer-order graph neural ODE models, along with several specific non-ODE-based methods, including GCN \cite{kipf2017semi}, GAT \cite{vaswani2017attention}, and GraphSAGE \cite{hamilton2017inductive}. It is worth noting that FROND models can be further integrated with other defense techniques, including adversarial training and pre-processing strategies. We delve into this aspect in the supplementary material.

\begin{figure*}[ht]
\centering
\begin{subfigure}{.3\textwidth}
    \centering
    \includegraphics[width=.95\linewidth]{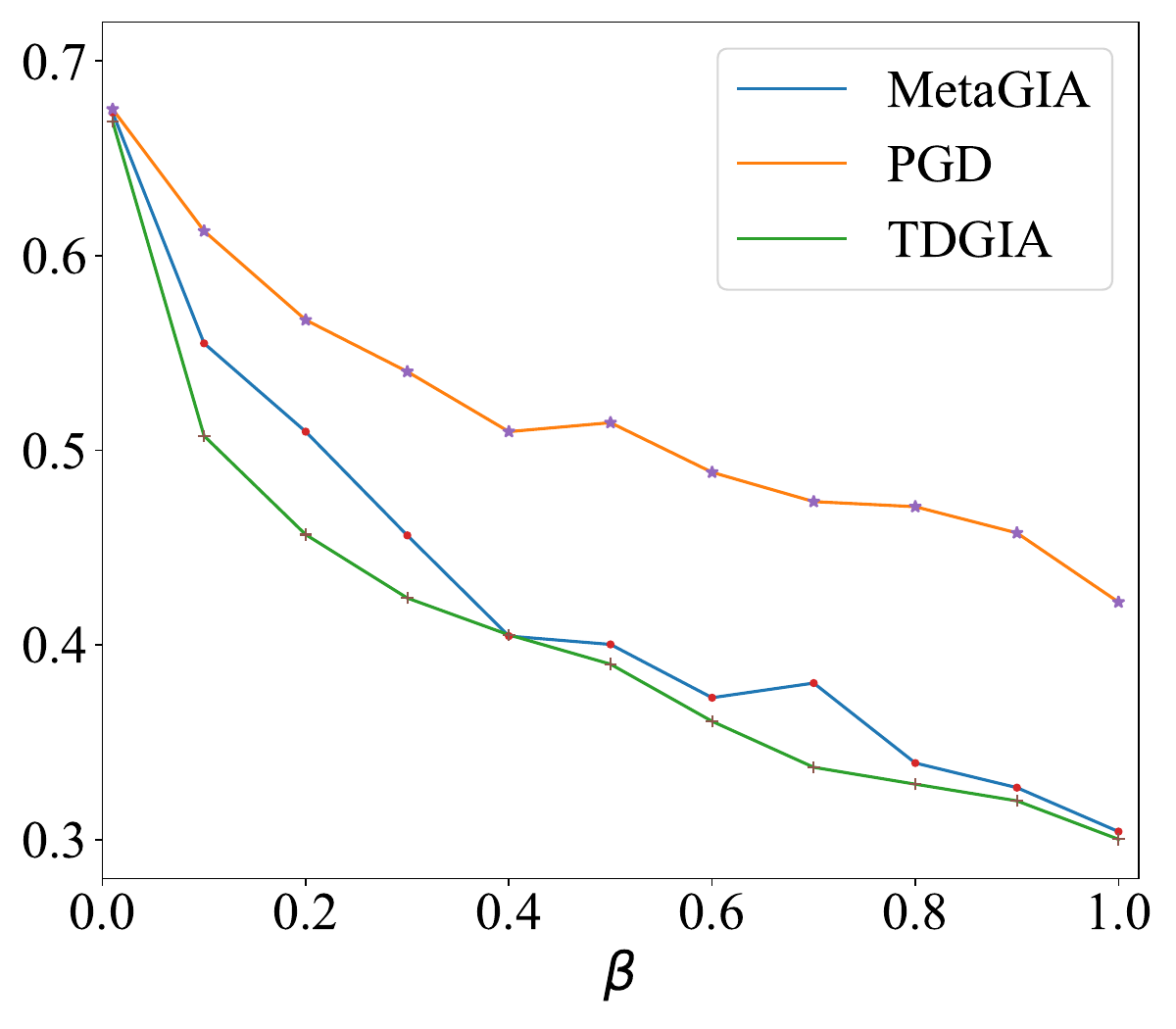} 
    \captionsetup{justification=centering}
    \caption{F-GRAND}
    \label{SUBFIGURE LABEL 1}
\end{subfigure}
\begin{subfigure}{.3\textwidth}
    \centering
    \includegraphics[width=.95\linewidth]{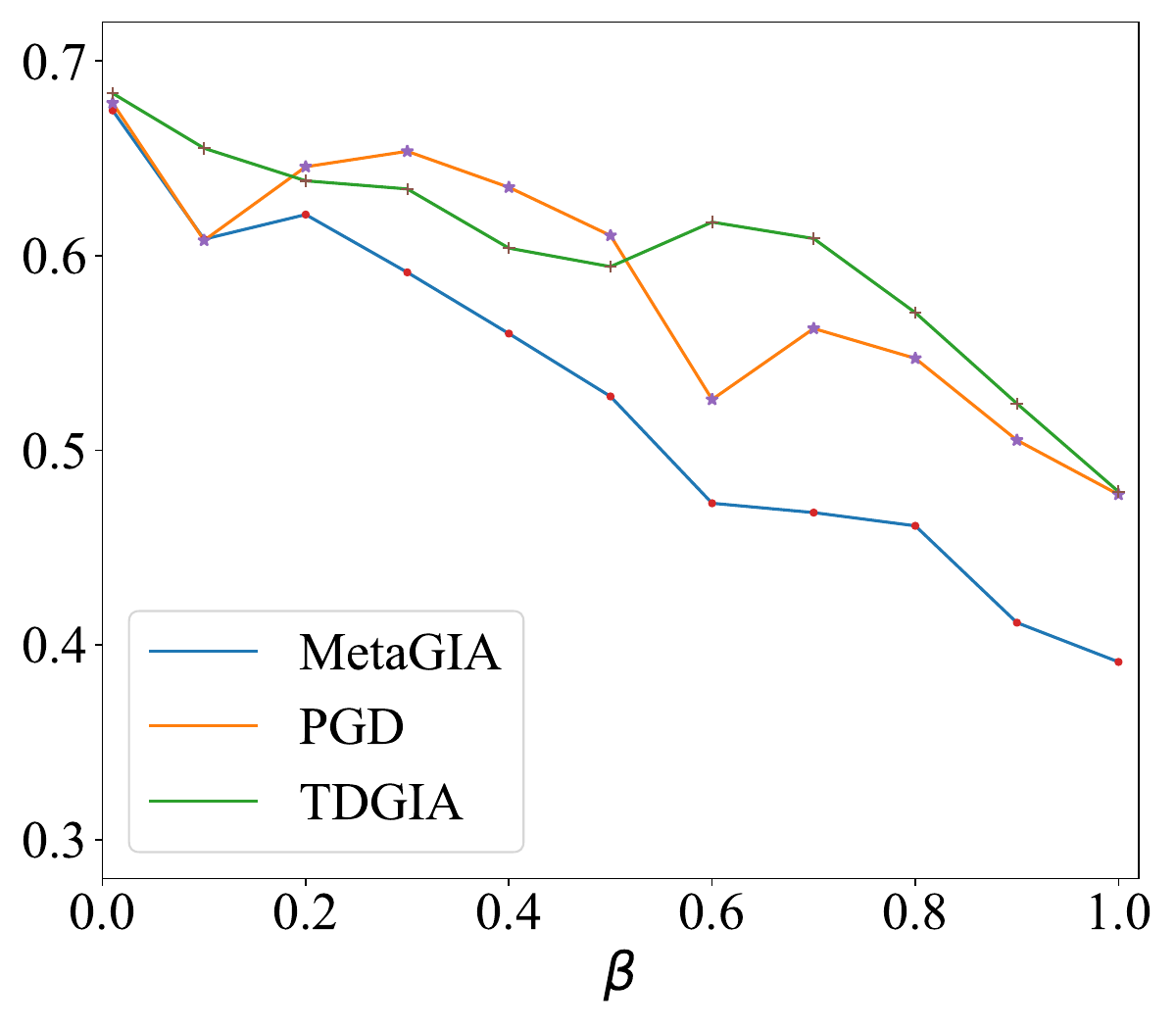}  
    \captionsetup{justification=centering}
    \caption{F-GraphBel}
    \label{SUBFIGURE LABEL 2}
\end{subfigure}
\begin{subfigure}{.3\textwidth}
    \centering
    \includegraphics[width=.95\linewidth]{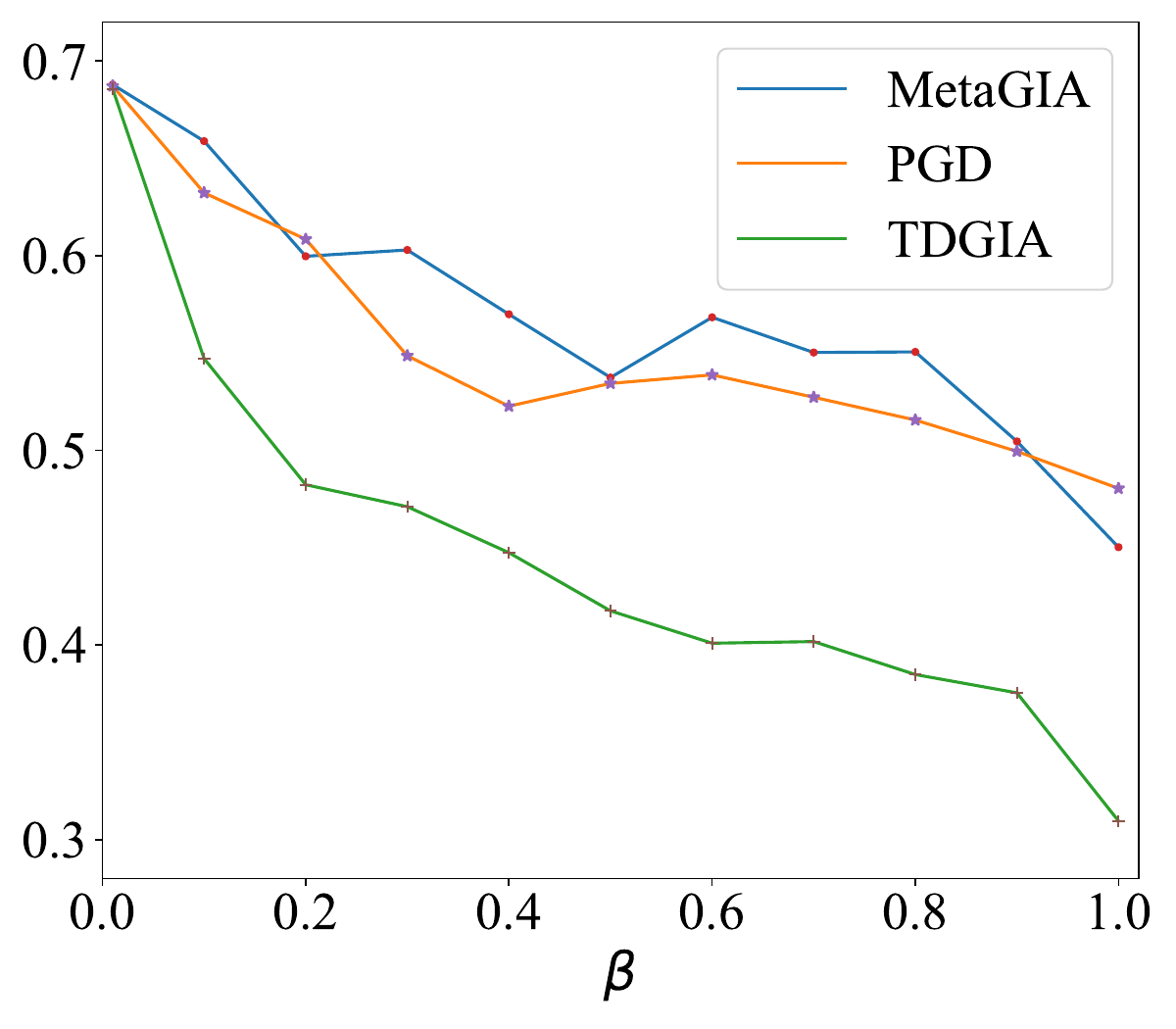} 
    \captionsetup{justification=centering}
    \caption{F-GraphCON}
    \label{SUBFIGURE LABEL 3}
\end{subfigure}
\caption{The impact of $\beta$ on the robust test accuracy.}
\label{fig:beta_model}
\end{figure*}
\subsection{GMA}
Our experimental setup involves the execution of graph modification adversarial attacks employing the Metattack method \cite{zugner_adversarial_2019}. Within the Metattack paradigm, the graph's adjacency matrix is perceived not just as a static structure but as a malleable hyperparameter. This perspective allows for attack optimization through meta-gradients to effectively address the inherent bilevel problem.
For the sake of ensuring a consistent and unbiased comparative landscape, our experiments strictly conform to the attack parameters as outlined in the paper \cite{jin2020prognn}. To achieve a comprehensive evaluation, we vary the perturbation rate, representing the proportion of edge modifications. We source the perturbed graph data from the comprehensive DeepRobust library \cite{li2020deeprobust}. The perturbation rate is adjusted in consistent increments of 5\%, starting from an untouched graph (0\%) and extending up to significant alterations at 25\%.

\subsection{GIA}
As elucidated in the paper \cite{chen2022hao}, GIA presents a considerably potent challenge to GNNs because of its ability to introduce new nodes and establish new edges within the original graph. Executing a GIA entails a two-step process: the injection of nodes and the subsequent update of features. During the node injection phase, new edges are established for the inserted nodes, driven by either gradient data or heuristic methods. 
Drawing inspiration from the methods proposed in \cite{chen2022hao}, we have incorporated three distinct GIA techniques: PGD-GIA, TDGIA\cite{zouKDD2021}, and MetaGIA.
The PGD-GIA method predominantly relies on a randomized approach for node injection. Once these nodes are in place, their features are meticulously curated using the Projected Gradient Descent (PGD) algorithm \cite{MadryICLR2018}. The Topological Deficiency Graph Injection Attack (TDGIA) \cite{zouKDD2021} exploits inherent topological weaknesses in graph structures. This approach harnesses these vulnerabilities to guide edge creation, optimizing a specific loss function to devise suitable features. MetaGIA \cite{chen2022hao} continually refines the adjacency matrix and node features, leaning heavily on gradient information to guide these refinements.
We conduct inductive learning for GIA in line with the data partitioning approach of the GRB framework \cite{zheng2021grb}, allocating 60\% for training, 10\% for validation, and 20\% for testing purposes. To maintain a balanced attack landscape, we pre-process the data using methods from \cite{zheng2021grb}, which involve excluding the 5\% of nodes with the lowest degrees (more susceptible to attacks) and the 5\% with the highest degrees (more resistant to attacks).

\begin{table*}[!htp]
\centering

\small
\makebox[\textwidth][c]{
\begin{tabular}{cccccccc} 
\toprule
Dataset & Attack & F-GRAND & GRAND  & F-GraphBel & GraphBel  & F-GraphCON & GraphCON   \\
\midrule

\multirow{3}{*}{Computers} 
& \emph{clean} &  90.0$\pm$0.05 &   \first{92.78$\pm$0.13}  & 88.36$\pm$1.05 & 90.14$\pm$0.27  & 89.99$\pm$0.28 &  \second{91.70$\pm$0.25} \\
& PGD & \second{75.29$\pm$1.17} &   16.44$\pm$0.11  &  \first{86.35$\pm$0.10} & 67.04$\pm$1.28   & 71.64$\pm$2.33 & 13.11$\pm$4.73 \\
& TDGIA & \second{71.99$\pm$0.73} &   15.10$\pm$0.76&  \first{86.21$\pm$0.21} & 53.75$\pm$2.84  & 66.35$\pm$1.94 &  4.33$\pm$4.21 \\

\bottomrule
\end{tabular}}
\caption{Node classification accuracy (\%) on graph {injection, evasion, non-targeted, white-box} attack  in {inductive} learning.  }
 \label{tab:adv_ind_white} 
\end{table*}

\subsection{Results}
\cref{tab:metattack} presents the results of GMA for transductive learning. As can be observed from the table, our proposed fractional-order methods outperform the original GRAND, GraphBel, and GraphCON in terms of robustness accuracy. These results validate and resonate with our theoretical findings, as discussed in \cref{theo:ml_mono}. Notably, these empirical observations underscore the capability of the FROND paradigm in enhancing a system's resilience, particularly when faced with input perturbations.
The GIA results are presented in \cref{tab:adv_trans_1}. We note that the fractional-order approach significantly improves post-attack accuracy relative to its integer-order graph neural ODE counterparts. Among these neural ODE models, \cite{SonKanWan:C22} demonstrated that GraphBel possesses superior robustness, which is further amplified by our fractional-order differential technique. 

\subsection{White-Box Attack}
White-box attacks, which directly target the model, are stronger than the black-box attacks used in \cref{tab:adv_trans_1}. To demonstrate that our graph neural FDE model can consistently improve the robustness of graph neural ODE models, we also conducted white-box GIA. The results are presented in \cref{tab:adv_ind_white}. Although the accuracy under white-box GIA is lower than under black-box GIA, our graph neural FDE models still outperform the graph neural ODE models. This observation aligns with our theoretical findings presented in \cref{sec.metho} of our main paper. Our graph neural FDE models indeed enhance the robustness of neural ODE models under both black-box and white-box scenarios.

\subsection{Ablation Study}
\subsubsection{Influence of $\beta$}
We assess the robustness accuracy of our fractional-order method across varying $\beta$ values. The findings are depicted in \cref{fig:beta_model}. A discernible trend emerges: as $\beta$ increases, the accuracy under the three GIA methods diminishes. This observation aligns with our theoretical insights presented in \cref{theo:ml_mono}.

\subsubsection{Model Complexity}
A comparison of inference times between our models and the baseline models is presented in \cref{tab:model_inference}. The results indicate that fractional-based models have similar inference times to graph neural ODE models. Notably, fractional-based models maintain the same training parameters as integer ODE models, avoiding any extra parameters. These findings highlight the efficiency and flexibility of our approach.

\begin{table}[!htp]
    \fontsize{9pt}{11pt}\selectfont
    \centering
    \makebox[0.4\textwidth][c]{
    \begin{tabular}{c|c}
    \toprule
    Model & Inf. Time(s) \\
    \midrule
    F-GRAND & 18.74 \\
    GRAND & 16.40 \\
    F-GraphBel & 64.90 \\
    GraphBel & 78.51 \\
    F-GraphCON & 21.33 \\
    GraphCON & 18.27 \\
    \bottomrule
    \end{tabular}}
    \caption{Inference time of models on the Cora dataset: integral time $T=10$ and step size of 1.}
    
    \label{tab:model_inference}
\end{table}

\section{Conclusion}
In this paper, we have undertaken a comprehensive exploration of robustness against adversarial attacks within the framework of the graph neural FDE models, i.e., FROND models, leading to substantial insights. Our investigation has yielded significant revelations, notably demonstrating the heightened robustness of the FROND models when compared to existing graph neural ODE models. Moreover, our work has contributed theoretical clarity, shedding light on the underlying reasons behind the heightened robustness of the FROND models in contrast to the graph neural ODE counterparts.


\section*{Acknowledgments}
This research is supported by the Singapore Ministry of Education Academic Research Fund Tier 2 grant MOE-T2EP20220-0002, and the National Research Foundation, Singapore and Infocomm Media Development Authority under its Future Communications Research and Development Programme. The computational work for this article was partially performed on resources of the National Supercomputing Centre, Singapore (https://www.nscc.sg). To improve the readability, parts of this paper have been grammatically revised using ChatGPT \cite{openai2022chatgpt4}.

\appendix 

This supplementary material complements the main body of our paper, providing additional details and supporting evidence for the assertions made therein. The structure of this document is as follows:

\begin{enumerate}
    \item A comprehensive background on fractional calculus is detailed in \cref{sec.review}.
    \item An expanded discussion on related work is provided in \cref{sec.more_rel}.
    \item Details of the fractional differential equation (FDE) solver used in our paper can be found in \cref{sec.supp_solver}.
        \item  Theoretical results from the main paper are rigorously proven in \cref{sec.proof}.
        \item Additional experiments, ablation studies, and dataset specifics are elaborated in \cref{sec.supp_exp}.
\end{enumerate}

\section{Review of Caputo Fractional Calculus}\label{sec.review}
We appreciate the need for a more accessible explanation of the Caputo time-fractional derivative and its derivation, as the mathematical intricacies may be challenging for some readers in the GNN community. To address this, we are providing a more comprehensive background in this section.

\subsection{Caputo Fractional Derivative and Its Compatibility of Integer-order Derivative}
The Caputo fractional derivative of a function $f(t)$ over an interval $[0,b]$, of a general positive order $\beta \in (0,\infty)$, is defined as follows:
\begin{align}
\label{Cap_Frac}
   D_t^\beta f(t)=\frac{1}{\Gamma(\lceil \beta \rceil-\beta)} \int_0^t(t-\tau)^{\lceil \beta \rceil-\beta-1} f^{(\lceil \beta \rceil)}(\tau) \mathrm{d} \tau,
\end{align}
Here, $\lceil \beta \rceil$ is the smallest integer greater than or equal to $\beta$, $\Gamma(\cdot)$ symbolizes the gamma function, and $f^{(\lceil \beta \rceil)}(\tau)$ signifies the $\lceil \beta \rceil$-order derivative of $f$. Within this definition, it is presumed that $f^{(\lceil \beta \rceil)} \in L^1[0,b]$, i.e., $f^{(\lceil \beta \rceil)}$ is Lebesgue integrable, to ensure the well-defined nature of $D_t^\beta f(t)$ as per \eqref{Cap_Frac} \cite{diethelm2010analysis}.
When addressing a vector-valued function, the Caputo fractional derivative is defined on a component-by-component basis for each dimension, similar to the integer-order derivative. For ease of exposition, we explicitly handle the scalar case here, although all following results can be generalized to vector-valued functions.
The Laplace transform for a general order $\beta \in (0,\infty)$ is presented in Theorem 7.1 \cite{diethelm2010analysis} as:
\begin{align}
\label{L_Cap_Frac}
\mathcal{L} D_t^\beta f(s)=s^\beta \mathcal{L} f(s)-\sum_{k=1}^{\lceil \beta \rceil} s^{\beta -k} f^{(k-1)}(0).
\end{align}
where we assume that $\mathcal{L}f$ exists on $[s_0,\infty)$ for some $s_0\in\mathbb{R}$. In contrast, for the integer-order derivative $f^{(\beta)}$ when $\beta$ is a positive integer, we also have the formulation \eqref{L_Cap_Frac}, with the only difference being the range of $\beta$. Therefore, as $\beta$ approaches some integer, the Laplace transform of the Caputo fractional derivative converges to the Laplace transform of the traditional integer-order derivative.  \emph{As a result, we can conclude that the Caputo fractional derivative operator generalizes the traditional integer-order derivative since their Laplace transforms coincide when $\beta$ takes an integer value.} Furthermore, the inverse Laplace transform indicates the uniquely determined $D_t^\beta=f^{(\beta)}$ (in the sense of almost everywhere \cite{Cohen2007}).

Under specific reasonable conditions, we can directly present this generalization as follows. We suppose $f^{(\lceil \beta \rceil)}(t)$ \eqref{Cap_Frac} is continuously differentiable. In this context, integration by parts can be utilized to demonstrate that
\begin{equation}
\begin{split}
D_t^\beta f(t) & =\frac{1}{\Gamma(\lceil\beta\rceil-\beta)}\Bigg(-\left[f^{(\lceil\beta\rceil)}(\tau) \frac{(t-\tau)^{\lceil\beta\rceil-\beta}}{\lceil\beta\rceil-\beta}\right]\bigg|_0^t\\
&\quad\quad+\int_0^t f^{(\lceil\beta\rceil+1)}(\tau) \frac{(t-\tau)^{\lceil\beta\rceil-\beta}}{\lceil\beta\rceil-\beta} \mathrm{d} \tau\Bigg) \\
& =\frac{t^{\lceil\beta\rceil-\beta} f^{(\lceil\beta\rceil)}(0)}{\Gamma(\lceil \beta \rceil-\beta+1)}+\frac{1}{\Gamma(\lceil \beta \rceil-\beta+1)} \\
&\quad\quad\times\int_{0}^{t}(t-\tau)^{\lceil \beta \rceil-\beta} f^{(\lceil\beta\rceil+1)}(\tau) \mathrm{d} \tau
\end{split}
\end{equation}
When $\beta\rightarrow \lceil \beta \rceil$, we get the following 
\begin{equation}
\begin{split}
\lim_{\beta\rightarrow \lceil \beta \rceil} D_t^\beta f(t) 
&= f^{(\lceil \beta \rceil)}(0)+\int_0^t f^{(\lceil\beta\rceil+1)}(\tau) \mathrm{d} \tau\\
&= f^{(\lceil \beta \rceil)} (0)+ f^{(\lceil \beta \rceil)}(t) -  f^{(\lceil \beta \rceil)}(0) \\
&= f^{(\lceil \beta \rceil)}(t)
\end{split}
\end{equation}
In parallel to the integer-order derivative, given certain conditions (\cite{diethelm2010analysis}[Lemma 3.13]), the Caputo fractional derivative possesses the semigroup property as illustrated in \cite{diethelm2010analysis}[Lemma 3.13]:
\begin{align}
D_t^{\varepsilon} D_t^n f= D_t^{n+\varepsilon} f. 
\end{align}
Nonetheless, it is crucial to recognize that, in general, the Caputo fractional derivative does not exhibit the semigroup property, a characteristic inherent to integer-order derivatives, as detailed in \cite{diethelm2010analysis}[Section 3.1]. The Caputo fractional derivative also exhibits linearity, but does not adhere to the same Leibniz and chain rules as its integer counterpart. As such properties are not utilized in our work, we refer interested readers to \cite{diethelm2010analysis}[Theorem 3.17 and Remark 3.5.]. We believe the above explanation facilitates understanding the relation between the Caputo derivative and its generalization of the integer-order derivative.

\section{More Discussion of Related Work: Fractional Calculus and Deep Learning} \label{sec.more_rel}
In this section, we further discussion the applications of fractional calculus, with a particular emphasis on its implications in deep learning.

Recently, fractional calculus has garnered significant attention due to its myriad applications spanning diverse areas. Key domains where fractional calculus has demonstrated potential include numerical analysis \cite{yuste2005explicit}, viscoelastic materials \cite{coleman1961foundations}, population dynamics \cite{almeida2016modeling}, control theory \cite{podlubny1994fractional}, signal processing \cite{machado2011recent}, financial mathematics \cite{scalas2000fractional}, and especially in characterizing porous and fractal systems \cite{nigmatullin1986realization,mandelbrot1982fractal,ionescu2017role}. Within these arenas, fractional-order differential equations have emerged as an enhanced alternative to their integer-ordered counterparts, serving as a robust mathematical tool for various system analyses \cite{diethelm2002analysis}. For instance, fractional calculus has been pivotal in diffusion process studies, elucidating phenomena from protein diffusion in cellular structures \cite{krapf2015mechanisms} to complex biological processes \cite{ionescu2017role}.

In the landscape of deep learning, \cite{liu2022regularized} introduced an innovative approach for GNN parameter optimization via fractional derivatives. This deviates significantly from the traditional use of integer-order derivatives in optimization algorithms such as SGD or Adam \cite{kingma2014adam}. However, the crux of their approach is fundamentally different from ours. While they focus on leveraging fractional derivatives for gradient optimization, our emphasis is on the fractional-derivative evolution of node embeddings. In another vein, \cite{antil2020fractional} draws from fractional calculus, specifically the L1 approximation of the Captou fractional derivative, to design a densely connected neural network. This design seeks to effectively manage non-smooth data and counter the vanishing gradient problem. Though our work orbits the same realm, our novelty lies in infusing fractional calculus into Graph ODE models, concentrating on the utility of fractional derivatives for evolving node embeddings, and highlighting its affinity with non-Markovian dynamic processes.

From the vantage of physics-informed machine learning, there exists a research trajectory dedicated to the formulation of neural networks anchored in physical principles, specifically tailored for solving fractional PDEs. A trailblazing contribution in this sphere is the Fractional Physics Informed Neural Networks (fPINNs) \cite{pang2019fpinns}. Subsequent explorations, including \cite{guo2022monte,wang2022fractional}, have expanded in this trajectory. It's pivotal to underline that these endeavors are distinctly different from our proposed methodology.

\section{Numerical Solvers for FDEs}\label{sec.supp_solver}
In this section, we present further details about how to solve FDEs using  the fractional Adams–Bashforth–Moulton method solvers from \cite{diethelm2004solvers}.
The predictor \(y^P_{k+1}\) is expressed as:
\begin{align}
  y^P_{k+1} = \sum_{j=0}^{\lceil\beta\rceil - 1} \frac{t^{j}_{k+1}}{j!} y^{(j)}_0 + \frac{1}{\Gamma(\beta)} \sum_{j=0}^{k} b_{j,k+1} f(t_j, y_j).  
\end{align}
Here, $k$ represents the current iteration or time step index in the discretization process. $h$ is the step size or time interval between successive approximations with $t_j=hj$ and $\lceil \cdot \rceil$ represents the ceiling function, when $0<\beta\le 1$, we have $\lceil \beta \rceil=1$. The coefficients $b_{j,k+1}$ are defined as follows:
\begin{equation}
    b_{j,k+1} = \frac{h^\beta}{\beta} \left( (k+1-j)^\beta - (k-j)^\beta \right),
\end{equation}

Leveraging this prediction, a corrector term can be formulated to enhance the solver's numerical accuracy. This can be viewed as the fractional counterpart of the traditional one-step Adams--Moulton method. However, we do not employ this additional corrector term in our paper. We reserve the examination of the corrector solution and its impact on GRAFDE for future work.

\section{Proof of \cref{theo:ml_mono}}\label{sec.proof}
We begin by referencing Eq.(127) from \cite{podlubny1994fractional}:
\begin{align}
  \begin{aligned} & E_{\beta, \alpha}(z)= 
                \frac{1}{\beta} z^{(1-\alpha) / \beta} \exp \left(z^{1 / \beta}\right)+ \\ 
  &\quad \quad \quad \frac{1}{2 \beta \pi i} \int_{\gamma(1, \varphi)} \frac{\exp \left(\zeta^{1 / \beta}\right) \zeta^{(1-\alpha) / \beta}}{\zeta-z} \ud \zeta, \\ 
  &\quad \quad \quad  \quad  \quad  \quad  \quad  \quad  \quad  \quad (|\arg (z)| \leq \varphi,|z|>1).\end{aligned}   \label{eq.gfgad}
\end{align}
Here, $ E_{\beta, \alpha}(z)$ denotes the generalized two-parameter Mittag-Leffler function, defined as $E_{\beta, \alpha}(z)\coloneqq\sum_{k=0}^{\infty} \frac{z^k}{\Gamma(\beta k+\alpha)}$. 
Its connection to the one-parameter Mittag-Leffler function, as given in \cref{def.ml}, is captured by $E_{\beta}(z) = E_{\beta, 1}(z)$. Note that here $z\in\Complex$ lies in the complex plane. 
The integral contour $\gamma(1, \varphi)$ is elaborately described in \cite[Figure 1.4, Sec 1.1.6]{podlubny1994fractional} and is displayed in \cref{fig.contour} for our reference. The parameter $\varphi$ is selected such that 
\begin{align}
    \frac{\pi \beta}{2}<\varphi<\min \{\pi, \pi \beta\}. \label{eq.dsfafzc}
\end{align}
For a thorough understanding of this integral contour, we direct readers to the aforementioned reference.

\begin{figure}[t]
    \centering
    \adjustbox{scale=0.8,center}{

\tikzset{every picture/.style={line width=0.75pt}} 

\begin{tikzpicture}[x=0.75pt,y=0.75pt,yscale=-1,xscale=1]

\draw  (155,169.6) -- (380.6,169.6)(265.06,39.8) -- (265.06,283.6) (373.6,164.6) -- (380.6,169.6) -- (373.6,174.6) (260.06,46.8) -- (265.06,39.8) -- (270.06,46.8)  ;
\draw  [draw opacity=0] (245.83,146.57) .. controls (250.98,142.27) and (257.58,139.66) .. (264.81,139.6) .. controls (281.38,139.47) and (294.92,152.79) .. (295.06,169.35) .. controls (295.19,185.92) and (281.87,199.46) .. (265.3,199.6) .. controls (257.87,199.66) and (251.05,197.01) .. (245.77,192.58) -- (265.06,169.6) -- cycle ; \draw   (245.83,146.57) .. controls (250.98,142.27) and (257.58,139.66) .. (264.81,139.6) .. controls (281.38,139.47) and (294.92,152.79) .. (295.06,169.35) .. controls (295.19,185.92) and (281.87,199.46) .. (265.3,199.6) .. controls (257.87,199.66) and (251.05,197.01) .. (245.77,192.58) ;  
\draw    (160.6,50.2) -- (245.83,146.57) ;
\draw [shift={(200.5,95.31)}, rotate = 48.51] [fill={rgb, 255:red, 0; green, 0; blue, 0 }  ][line width=0.08]  [draw opacity=0] (7.14,-3.43) -- (0,0) -- (7.14,3.43) -- (4.74,0) -- cycle    ;
\draw    (160.2,290.6) -- (245.77,192.58) ;
\draw [shift={(204.69,239.63)}, rotate = 131.12] [fill={rgb, 255:red, 0; green, 0; blue, 0 }  ][line width=0.08]  [draw opacity=0] (7.14,-3.43) -- (0,0) -- (7.14,3.43) -- (4.74,0) -- cycle    ;
\draw  [draw opacity=0][dash pattern={on 4.5pt off 4.5pt}] (239.34,138.34) .. controls (246.26,133.05) and (255.02,129.86) .. (264.57,129.78) .. controls (287.21,129.59) and (305.7,146.98) .. (305.88,168.62) .. controls (305.88,169.11) and (305.88,169.61) .. (305.86,170.09) -- (264.89,168.96) -- cycle ; \draw [color={rgb, 255:red, 155; green, 155; blue, 155 }  ,draw opacity=1 ][dash pattern={on 4.5pt off 4.5pt}] [dash pattern={on 4.5pt off 4.5pt}]  (241.8,136.6) .. controls (248.29,132.36) and (256.12,129.85) .. (264.57,129.78) .. controls (287.21,129.59) and (305.7,146.98) .. (305.88,168.62) .. controls (305.88,169.11) and (305.88,169.61) .. (305.86,170.09) ;  \draw [shift={(239.34,138.34)}, rotate = 331.02] [fill={rgb, 255:red, 155; green, 155; blue, 155 }  ,fill opacity=1 ][dash pattern={on 3.49pt off 4.5pt}][line width=0.08]  [draw opacity=0] (3.57,-1.72) -- (0,0) -- (3.57,1.72) -- cycle    ;
\draw  [draw opacity=0][dash pattern={on 4.5pt off 4.5pt}] (301,169.76) .. controls (301.06,188.94) and (285.29,204.59) .. (265.69,204.75) .. controls (257,204.82) and (249.01,201.84) .. (242.79,196.82) -- (265.41,169.88) -- cycle ; \draw [color={rgb, 255:red, 128; green, 128; blue, 128 }  ,draw opacity=1 ][dash pattern={on 4.5pt off 4.5pt}] [dash pattern={on 4.5pt off 4.5pt}]  (301,169.76) .. controls (301.06,188.94) and (285.29,204.59) .. (265.69,204.75) .. controls (258.04,204.82) and (250.94,202.51) .. (245.11,198.54) ; \draw [shift={(242.79,196.82)}, rotate = 37.35] [fill={rgb, 255:red, 128; green, 128; blue, 128 }  ,fill opacity=1 ][dash pattern={on 3.49pt off 4.5pt}][line width=0.08]  [draw opacity=0] (3.57,-1.72) -- (0,0) -- (3.57,1.72) -- cycle    ; 
\draw [color={rgb, 255:red, 128; green, 128; blue, 128 }  ,draw opacity=1 ] [dash pattern={on 4.5pt off 4.5pt}]  (264.89,168.96) -- (280.62,147.42) ;
\draw [shift={(281.8,145.8)}, rotate = 126.13] [color={rgb, 255:red, 128; green, 128; blue, 128 }  ,draw opacity=1 ][line width=0.75]    (4.37,-1.32) .. controls (2.78,-0.56) and (1.32,-0.12) .. (0,0) .. controls (1.32,0.12) and (2.78,0.56) .. (4.37,1.32)   ;

\draw (160.8,248.4) node [anchor=north west][inner sep=0.75pt]  [font=\tiny]  {$\gamma ( 1,\varphi )$};
\draw (284.4,122.8) node [anchor=north west][inner sep=0.75pt]  [font=\tiny,color={rgb, 255:red, 128; green, 128; blue, 128 }  ,opacity=1 ]  {$\varphi $};
\draw (280.4,202.4) node [anchor=north west][inner sep=0.75pt]  [font=\tiny,color={rgb, 255:red, 128; green, 128; blue, 128 }  ,opacity=1 ]  {$-\varphi $};
\draw (274.4,157.4) node [anchor=north west][inner sep=0.75pt]  [font=\tiny,color={rgb, 255:red, 128; green, 128; blue, 128 }  ,opacity=1 ]  {$1$};

\end{tikzpicture}

}
    \caption{Contour $\gamma(1, \varphi)$}
    \label{fig.contour}
\end{figure}
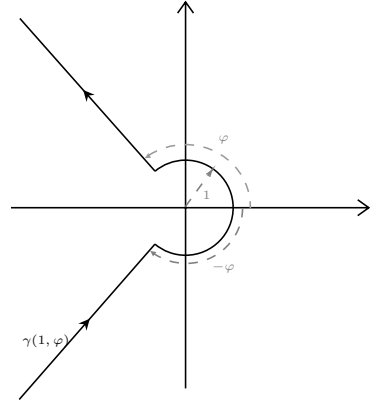

Upon substituting $z=L T^{\beta}$ into \cref{eq.gfgad} with $\alpha=1$, we obtain 
\begin{align}
\begin{aligned}
    E_{\beta}&(L T^{\beta}) = \frac{1}{\beta} \exp \left(L^{1 / \beta} T\right)+\frac{1}{2 \beta \pi i} \int_{\gamma(1, \varphi)} \frac{\exp \left(\zeta^{1 / \beta}\right) }{\zeta-LT^{\beta}} \ud \zeta, \label{eq.gadd}
\end{aligned}
\end{align}
assuming $T$ is sufficiently large such that $L T^\beta>1$, thereby satisfying the condition stated in \cref{eq.gfgad}.

\noindent\tb{First Term Analysis:}
For the first term $\frac{1}{\beta} \exp \left( L^{1 / \beta} T\right)$, the derivative with respect to $
\beta$ is 
\begin{align*}
w(\beta)&\coloneqq-\frac{\exp(L^{1 / \beta} T)}{\beta^2}-\frac{\exp(L^{1 / \beta} T) L^{1 / \beta} T \ln (L)}{\beta^3} \\
&= -\frac{\exp(L^{1 / \beta} T)}{\beta^3}\left(\beta+L^{1 / \beta} T \ln (L)\right)
\end{align*}

We have that when $T$ is sufficiently large, $\beta+L^{1 / \beta} T \ln (L)$  becomes negative due to $L < 1$. This observation directly implies that $w(\beta)$ is positive under these conditions. Moreover, for $\beta$ within the interval $[\varepsilon, 1]$, we observe that:
\begin{align}
    w(\beta)>-{\exp(L^{1 / \beta} T)}\left(1+L^{1 / \beta} T \ln (L)\right), \label{eq.tdsafadc}
\end{align}
which is unbounded when $T\rightarrow\infty$.\\

\noindent\tb{Second Term Analysis:}
We denote the integral in \cref{eq.gadd} as
\begin{align*}
I(\beta)\coloneqq\frac{1}{2 \beta \pi i} \int_{\gamma(1, \varphi)} \frac{\exp \left(\zeta^{1 / \beta}\right) }{\zeta-LT^{\beta}} \ud \zeta.
\end{align*}
We now take the derivative of $I(\beta)$ with respect to $\beta$ and obtain 
\begin{align}
    \begin{aligned}
     I'(\beta) &= \frac{1}{2 \beta \pi i} \int_{\gamma(1, \varphi)} \frac{\zeta^{\frac{1}{\beta}}\ln (\zeta) \exp \left(\zeta^{\frac{1}{\beta}}\right)}{\beta^2\left(L T^\beta-\zeta\right)}+\\
    &\quad \quad \quad \frac{L T^\beta \ln (T)\exp \left(\zeta^{\frac{1}{\beta}}\right)}{\left(L T^\beta-\zeta\right)^2} \ud \zeta. \label{eq.rfadfa}
    \end{aligned}
\end{align}
For the first item in \cref{eq.rfadfa}, we denote it as $ H_1(T)$:
\begin{align}
\begin{aligned}
   H_1(T)&\coloneqq \frac{1}{2 \beta \pi i} \int_{\gamma(1, \varphi)} \frac{\zeta^{\frac{1}{\beta}}\ln (\zeta) \exp \left(\zeta^{\frac{1}{\beta}}\right)}{\beta^2\left(L T^\beta-\zeta\right)}\ud \zeta 
\end{aligned}  
\end{align}
\begin{figure}[H]
    \centering
    \adjustbox{scale=0.8,center}{

\tikzset{every picture/.style={line width=0.75pt}} 

\begin{tikzpicture}[x=0.75pt,y=0.75pt,yscale=-1,xscale=1]

\draw  (155,169.6) -- (380.6,169.6)(265.06,39.8) -- (265.06,283.6) (373.6,164.6) -- (380.6,169.6) -- (373.6,174.6) (260.06,46.8) -- (265.06,39.8) -- (270.06,46.8)  ;
\draw  [draw opacity=0] (277.03,142.09) .. controls (287.56,146.67) and (294.95,157.13) .. (295.06,169.35) .. controls (295.16,181.82) and (287.64,192.57) .. (276.85,197.19) -- (265.06,169.6) -- cycle ; \draw   (277.03,142.09) .. controls (287.56,146.67) and (294.95,157.13) .. (295.06,169.35) .. controls (295.16,181.82) and (287.64,192.57) .. (276.85,197.19) ;  
\draw    (320.2,61) -- (277.03,142.09) ;
\draw [shift={(300.54,97.92)}, rotate = 118.03] [fill={rgb, 255:red, 0; green, 0; blue, 0 }  ][line width=0.08]  [draw opacity=0] (7.14,-3.43) -- (0,0) -- (7.14,3.43) -- (4.74,0) -- cycle    ;
\draw    (319.8,279.8) -- (276.85,197.19) ;
\draw [shift={(297.13,236.19)}, rotate = 62.53] [fill={rgb, 255:red, 0; green, 0; blue, 0 }  ][line width=0.08]  [draw opacity=0] (7.14,-3.43) -- (0,0) -- (7.14,3.43) -- (4.74,0) -- cycle    ;
\draw  [draw opacity=0][dash pattern={on 4.5pt off 4.5pt}] (280.96,132.89) .. controls (295.5,138.81) and (305.75,152.55) .. (305.88,168.62) .. controls (305.88,169.11) and (305.88,169.61) .. (305.86,170.09) -- (264.89,168.96) -- cycle ; \draw [color={rgb, 255:red, 155; green, 155; blue, 155 }  ,draw opacity=1 ][dash pattern={on 4.5pt off 4.5pt}] [dash pattern={on 4.5pt off 4.5pt}]  (283.74,134.15) .. controls (296.79,140.61) and (305.76,153.59) .. (305.88,168.62) .. controls (305.88,169.11) and (305.88,169.61) .. (305.86,170.09) ;  \draw [shift={(280.96,132.89)}, rotate = 30.53] [fill={rgb, 255:red, 155; green, 155; blue, 155 }  ,fill opacity=1 ][dash pattern={on 3.49pt off 4.5pt}][line width=0.08]  [draw opacity=0] (3.57,-1.72) -- (0,0) -- (3.57,1.72) -- cycle    ;
\draw    (392.6,151.8) -- (294.63,165.29) ;
\draw [shift={(347.68,157.98)}, rotate = 172.16] [fill={rgb, 255:red, 0; green, 0; blue, 0 }  ][line width=0.08]  [draw opacity=0] (7.14,-3.43) -- (0,0) -- (7.14,3.43) -- (4.74,0) -- cycle    ;
\draw  [dash pattern={on 0.84pt off 2.51pt}]  (264.89,168.96) -- (297.8,165) ;
\draw    (357,157) -- (358.6,169.8) ;
\draw   (354.21,157.22) -- (356.78,156.81) -- (357.19,159.38) -- (354.62,159.79) -- cycle ;
\draw  [draw opacity=0][dash pattern={on 4.5pt off 4.5pt}] (319.1,161.69) .. controls (319.69,164.11) and (320.05,166.63) .. (320.17,169.2) -- (280.95,170.97) -- cycle ; \draw [color={rgb, 255:red, 155; green, 155; blue, 155 }  ,draw opacity=1 ][dash pattern={on 4.5pt off 4.5pt}] [dash pattern={on 4.5pt off 4.5pt}]  (319.69,164.64) .. controls (319.94,166.13) and (320.1,167.66) .. (320.17,169.2) ;  \draw [shift={(319.1,161.69)}, rotate = 85.16] [fill={rgb, 255:red, 155; green, 155; blue, 155 }  ,fill opacity=1 ][dash pattern={on 3.49pt off 4.5pt}][line width=0.08]  [draw opacity=0] (3.57,-1.72) -- (0,0) -- (3.57,1.72) -- cycle    ;

\draw (284.4,264.8) node [anchor=north west][inner sep=0.75pt]  [font=\tiny]  {$\gamma ( 1,\varphi )$};
\draw (308.8,122.8) node [anchor=north west][inner sep=0.75pt]  [font=\tiny,color={rgb, 255:red, 128; green, 128; blue, 128 }  ,opacity=1 ]  {$\varphi $};
\draw (323.6,161.4) node [anchor=north west][inner sep=0.75pt]  [font=\tiny]  {$\delta $};

\end{tikzpicture}

}
    \caption{}
    \label{fig.ffadad}
\end{figure}
We then deduce that for large $T$, the minimum distance from the point $(LT^\beta, 0)$ to the contour $\gamma(1, \varphi)$  satisfies
\begin{align*}
\min_{\zeta \in \gamma(1, \varphi)}|\zeta - LT^\beta| \ge p LT^\beta,
\end{align*}
where $p > 0$ is a leading constant. 
This is due to the fact that, with $\beta$ ranging within $[\varepsilon, 1]$ where $\varepsilon > 0$, as depicted in \cref{fig.ffadad}, a sufficiently small $\delta > 0$ can be identified such that $\min_{\zeta \in \gamma(1, \varphi)}|\zeta - LT^\beta| \ge \sin(\delta) LT^\beta$. Consequently, we can set $p = \sin(\delta)$.  We then have that
\begin{align}
\begin{aligned}
   &\abs{H_1(T)}\\
  \le  &\frac{1}{2 p \beta \pi LT^\beta} \int_{\gamma(1, \varphi)} \abs{\zeta^{\frac{1}{\beta}}\ln (\zeta) \exp \left(\zeta^{\frac{1}{\beta}}\right)\beta^{-2}}\ud \zeta \\
     =  & \frac{1}{2 p\beta^3 \pi LT^\beta} \int_{\gamma(1, \varphi)} \abs{\zeta^{\frac{1}{\beta}}\ln (\zeta)} \abs{\exp \left(\zeta^{\frac{1}{\beta}}\right)}\ud \zeta \\
\end{aligned}  
\end{align}
The integral on the right-hand side converges, because for $\zeta$ such that $\arg (\zeta)= \pm \varphi$ and $|\zeta| \geq 1$ the following holds:
\begin{align*}
\left|\exp \left(\zeta^{1 / \beta}\right)\right|=\exp \left(|\zeta|^{1 / \beta} \cos \left(\frac{\varphi}{\beta}\right)\right)
\end{align*}
where $\cos (\varphi / \beta)<0$ due to condition in \cref{eq.dsfafzc}. Denote its value as $C_1(\beta)$ which is continuous w.r.t $\beta$.
Thus, we have
\begin{align}
\begin{aligned}
   \abs{H_1(T)}\le \frac{1}{2 p\beta^3 \pi LT^\beta} C_1(\beta)\\
\end{aligned}  
\end{align}
Since $\beta$ takes value in the compact set $[\varepsilon,1]$, the continuous function $\frac{C_1(\beta)}{\beta^3}$ is bounded. Therefore, we conclude that $|H_1(T)| < 1$ when $T$ is sufficiently large for all $\beta\in [\varepsilon,1]$.

Similarly, for the second term in \cref{eq.rfadfa}, denoted as $H_2(T)$, we have
\begin{align}
\begin{aligned}
   \abs{H_2(T)} &\le  \frac{L T^\beta \ln (T)}{2 p^2\beta \pi L^2T^{2\beta}} \int_{\gamma(1, \varphi)} \abs{\exp \left(\zeta^{\frac{1}{\beta}}\right)}\ud \zeta \\
   & \le  \frac{\ln (T)}{2 p^2\beta \pi L T^{\beta}} \int_{\gamma(1, \varphi)} \abs{\exp \left(\zeta^{\frac{1}{\beta}}\right)}\ud \zeta \\
\end{aligned}  
\end{align}
Following the same reasoning, we deduce that $|H_2(T)| < 1$ when $T$ is sufficiently large for all $\beta\in [\varepsilon,1]$. Recalling from \cref{eq.tdsafadc} that $w(\beta)$ is unbounded, the proof is completed by combining the terms to conclude the derivative of $E_{\beta}(L T^{\beta})$ with respect to $\beta$ is positive over $[\varepsilon, 1]$ for sufficiently large $T$.

\section{Experimets} \label{sec.supp_exp}

\subsection{Datasets}

The statistics of the datasets used in our experiments are presented in \cref{tab:data}. The attack budgets for the GIA, as shown in \cref{tab:adv_trans_1}, are detailed in \cref{tab:attack_max_min}. We adhered to the attack budgets specified in the paper \cite{chen2022hao} for GIA.

\begin{table}[!htp]
    \small
    \centering
    \begin{tabular}{ccccccc} 
    \toprule
    Dataset &  \# Nodes & \# Edges & \# Features & \# Classes  \\
    \midrule
     Cora 	   & 2708  & 5429  & 1433  & 7   \\
    \midrule
     Citeseer 	   & 3327  & 4732	  & 3703  & 6   \\
    \midrule
     PubMed 		   & 19717  & 44338	  &  500  & 3  \\
    \midrule
   
    Computers 	   & 13,752  & 245,861	  & 767  & 10    \\

    \bottomrule
    \end{tabular}
    \caption{Dataset Statistics}
    \label{tab:data}
\end{table}

\subsection{White-box attack}

White-box attacks, which directly target the model, are stronger than the black-box attacks used in \cref{tab:adv_trans_1}. To demonstrate that FROND can consistently improve the robustness of graph neural ODE models, we also conducted white-box GIA. The results are presented in \cref{tab:adv_ind_white}. Although the accuracy under white-box GIA is lower than under black-box GIA, our FROND models still outperform the graph neural ODE models. This observation aligns with our theoretical findings presented in \cref{sec.metho} of our main paper. Our FROND models indeed enhance the robustness of neural ODE models under both black-box and white-box scenarios.
\begin{table}[]
\small
    \centering
   \begin{tabular}{ccc} 
    \toprule
    Dataset & max \# Nodes & max \# Edges \\
    \midrule
    
     Cora &  60  & 20   \\
    \midrule
    
     Citeseer &  90  & 10   \\
    \midrule
    
     PubMed &  200  & 100   \\
    \midrule
    
      Computers &  300  & 150   \\
    
    \bottomrule
    \end{tabular}
    \caption{Attack budget for GIA}
    \label{tab:attack_max_min}
\end{table}

\subsection{Adversarial training}

Adversarial training (AT), as demonstrated in the paper \cite{madry2017pgd_atraining}, is an effective strategy for mitigating attacks. It entails the incorporation of perturbations or noise during the training process to bolster the robustness of the model. This approach serves as a general framework that can be applied to any model. In this work, we employ the Projected Gradient Descent (PGD) adversarial training method (AT-PGD), as outlined in \cite{madry2017pgd_atraining}, to train our graph neural FDE models and enhance their robustness. 

The results of AT are presented in \cref{tab:adv_ind_at}. It is evident that AT-PGD significantly improves the robustness of FROND models. These findings demonstrate that FROND models can be effectively combined with other defense mechanisms.

\begin{table*}[t]
\centering

\small
\makebox[\textwidth][c]{
\begin{tabular}{cccccccc} 
\toprule
Dataset & Attack & F-GRAND-AT & F-GraphBel-AT  & F-GraphCON-AT   \\

\midrule
\multirow{4}{*}{Citeseer} 
& \emph{clean} &  72.0$\pm$0.31 & 71.37$\pm$0.80 & 65.99$\pm$0.16 \\
& PGD &  71.16$\pm$0.84 & 71.25$\pm$0.24 & 65.93$\pm$0.10 \\
& TDGIA &  71.26$\pm$0.73 &  70.01$\pm$0.65 & 64.94$\pm$0.26 \\
& MetaGIA &  71.58$\pm$0.63 & 70.95$\pm$0.85 & 65.73$\pm$0.42 \\
\bottomrule
\end{tabular}}
\caption{Node classification accuracy (\%) on graph {\bf injection, evasion, non-targeted, black-box} attack  in {\bf inductive} learning.}
 \label{tab:adv_ind_at} 
\end{table*}

\subsection{Pre-processing methods}
As mentioned in our main paper, several methods \cite{zhang2020gnnguard}, \cite{jin2020prognn}, \cite{deng2022garnet} employ preprocessing techniques to prune or rewire the graph structure, thereby removing malicious edges or nodes. In this work, we integrate FROND models with GNNGUARD \cite{zhang2020gnnguard} to further enhance their performance. GNNGUARD identifies suspicious nodes or edges and refines the edge weights to mitigate the influence of these suspicious edges.

We introduce the models F-GraphCON-GUARD and GraphCON-GUARD by incorporating the GNNGUARD preprocessing techniques into the graph neural FDE model and neural ODE model, respectively. The results after adversarial attacks are presented in \cref{tab:adv_ind_guard}. As can be observed, GNNGUARD enhances the robust accuracy in both cases. However, our F-GraphCON-GUARD performs better, as the base F-GraphCON model already exhibits superior robustness compared to GraphCON. This demonstrates that our graph neural FDE model can be seamlessly integrated with preprocessing methods to further improve robustness.

\begin{table*}[!h]
\centering

\small
\makebox[\textwidth][c]{
\begin{tabular}{cccccccc} 
\toprule
Dataset & Attack & F-GraphCON-GUARD & F-GraphCON  & GraphCON-GUARD  & GraphCON   \\
\midrule

\multirow{4}{*}{Citeseer} 
& \emph{clean} & 71.22$\pm$0.61 &   73.50$\pm$0.34  & 71.93$\pm$0.82 & 72.07$\pm$0.93 \\
& PGD &  62.28$\pm$1.86 &   54.47$\pm$1.0  & 46.52$\pm$3.37 & 37.71$\pm$7.0 \\
& TDGIA &  63.36$\pm$1.34 &   54.71$\pm$1.69  & 51.52$\pm$1.71 & 30.93$\pm$3.0 \\
& MetaGIA &  59.73$\pm$2.26  &   48.82$\pm$3.27  & 51.70$\pm$3.40 & 29.09$\pm$2.01 \\
\midrule

\multirow{4}{*}{Pubmed} 
& \emph{clean} & 90.06$\pm$0.07 &  90.30$\pm$0.11  & 89.60$\pm$0.17 & 
 88.09$\pm$0.32  \\
& PGD & 83.39$\pm$3.38 &   51.16$\pm$6.04  &  63.55$\pm$2.19 & 45.85$\pm$1.97  \\
& TDGIA & 73.20$\pm$1.44 &   55.50$\pm$4.03  & 51.73$\pm$2.94 & 45.57$\pm$2.02  \\
& MetaGIA & 79.97$\pm$5.01 &   52.03$\pm$5.53 & 62.44$\pm$3.71  & 45.81$\pm$2.81  \\

\bottomrule
\end{tabular}}
\caption{Node classification accuracy (\%) on graph {\bf injection, evasion, non-targeted} attack  in {\bf inductive} learning.  }
 \label{tab:adv_ind_guard} 
\end{table*}

\bibliography{bib/IEEEabrv, bib/aaai24}

\end{document}